%% file: main.tex
\documentclass{article}

\PassOptionsToPackage{numbers,square,comma}{natbib}
\usepackage[main,final]{neurips_2026}

\usepackage[utf8]{inputenc} 
\usepackage[T1]{fontenc}    
\usepackage{hyperref}       
\usepackage{url}           
\usepackage{booktabs}       
\usepackage{amsfonts}       
\usepackage{xspace}         

\newcommand{\repo}{\href{https://github.com/Wang-ML-Lab/sae-plus-plus}{\textcolor{black}{\underline{https://github.com/Wang-ML-Lab/sae-plus-plus}}}}
\usepackage{nicefrac}       
\usepackage{microtype}      
\usepackage{xcolor}         
\usepackage{graphicx}
\usepackage{amsmath,amssymb,enumitem, amsthm}
\usepackage{cleveref}
\usepackage{multirow}
\usepackage{wrapfig}
\usepackage{float}
\newcommand{\modelname}{SAE++\xspace}
\newcommand{\modelnames}{SAE++\xspace}

\input{defs.tex}

\title{SAE++: Cascaded Sparse Autoencoders Learn\\Multi-Level Visual Concepts in Multimodal LLMs}



\author{%
  Yusong Zhao$^{1,\ast,\dagger,\ddagger}$
  \And
  Hengyi Wang$^{2,\ast}$
  \And
  Tanuja Ganu$^{3}$
  \And
  Akshay Nambi$^{3}$
  \And
  Hao Wang$^{1}$
}

\begin{document}

\maketitle

\begingroup
\renewcommand{\thefootnote}{}
\footnotetext{%
$^{1}$University of Illinois Urbana-Champaign.
\quad
$^{2}$Rutgers University.
\quad
$^{3}$Microsoft Research.
\quad
$^{\ast}$Co-first authors.
\quad
$^{\dagger}$Project lead.
\quad
$^{\ddagger}$Correspondence to: Yusong Zhao
\texttt{<yusongz@illinois.edu>}.
}
\endgroup

\begin{abstract}
Multimodal Large Language Models (MLLMs) have demonstrated strong performance on vision-language tasks, yet their internal visual representations remain difficult to interpret.
Sparse Autoencoders (SAEs) provide a scalable way to decompose dense model activations into sparse, interpretable features.
However, existing SAE architectures primarily recover flat feature dictionaries and are less suited for explicit multi-level concept organization.
In this paper, we introduce a cascaded sparse autoencoder architecture, dubbed \modelnames{}, for learning hierarchical visual concepts in MLLMs.
{Rather than nesting or stacking SAE sparse activation codes, \modelnames{} {trains} a second-level SAE directly on the decoder weights of the first-level SAE, treating learned low-level feature directions as inputs for higher-level abstraction.}
{This design enables \modelnames{} to learn ``concepts of concepts'' while avoiding drawbacks from the shared-prefix coupling of nesting, Matryoshka-style hierarchies and the bottlenecks of naively stacked SAEs.}
Experiments across Qwen3-VL, Gemma-3, and LLaVA on multiple visual datasets show that \modelnames{} {improves} interpretability in terms of hierarchical concept coherence over state-of-the-art SAE baselines.
Results on concept steering further demonstrate that the learned concept groups support effective group-level interventions in MLLM outputs. {Code is available at \repo.}
\end{abstract}

\section{Introduction}
Multimodal Large Language Models (MLLMs) such as LLaVA \cite{liu2023visualinstructiontuning}, Qwen-VL \cite{bai2023qwenvlversatilevisionlanguagemodel}, and Gemma \cite{gemmateam2024gemmaopenmodelsbased} have demonstrated exceptional capabilities in processing visual inputs beyond the purely textual scope of traditional LLMs. They bridge the gap between advanced visual perception and the reasoning abilities of Large Language Models, achieving state-of-the-art performance in complex tasks ranging from visual question answering and embodied agency to general-purpose personal assistance \cite{driess2023palmeembodiedmultimodallanguage, moor2023foundation}. Despite their rapid deployment in commercial and everyday applications, most of these models remain ``black boxes'', often exhibiting unpredictable and hazardous behaviors. These include severe hallucinations, such as confidently describing non-existent objects \cite{bai2025hallucinationmultimodallargelanguage, li2023evaluatingobjecthallucinationlarge}, and vulnerability to visual jailbreak attacks that bypass safety alignment \cite{qi2023visualadversarialexamplesjailbreak, ma2024visualroleplayuniversaljailbreakattack}. 
These risks highlight the need for interpretable representations and 
predictions~\cite{CounTS}, particularly the concepts encoded inside MLLMs.

However, the mechanistic interpretation of MLLMs presents challenges that exceed those of traditional deep learning models. A primary obstacle is the intrinsic polysemy of deep neural networks, where a single neuron may activate for disparate concepts \cite{elhage2022superposition, gandelsman2024interpretingclipsimagerepresentation}. Moreover, the high-dimensional nature of internal MLLM layers further exacerbates this issue. Given the open-ended behavior of MLLMs, disentangling these features through manual annotation is infeasible.

To address these challenges, Sparse Autoencoders (SAEs) \cite{bricken2023monosemanticity} have emerged as a scalable, unsupervised methodology. By projecting dense model activations into an overcomplete, sparse latent space, SAEs explicitly attempt to decompose superimposed features into distinct, monosemantic directions. SAEs have been successfully used to discover diverse features in both large language models~\cite{cunningham2023sparseautoencodershighlyinterpretable} and vision-language models~\cite{lim2025sparseautoencodersrevealselective, zaigrajew2025interpretingcliphierarchicalsparse, lou2025saevinterpretingmultimodalmodels, shen2025vlsaeinterpretingenhancingvisionlanguage, papadimitriou2025interpretinglinearstructurevisionlanguage}. Recent multimodal SAE studies further show that sparse features can help analyze vision-language alignment, shared embedding spaces, and data quality in MLLMs. {Despite these advances, learning explicit multi-level visual concept hierarchies from MLLMs remains challenging.}

Recent work attempts to address this issue using Matryoshka SAEs inspired by Matryoshka representation learning \cite{kusupati2024matryoshkarepresentationlearning, bussmann2025learningmultilevelfeaturesmatryoshka, zaigrajew2025interpretingcliphierarchicalsparse}, discovering hierarchical features in models such as CLIP \cite{radford2021learningtransferablevisualmodels} and Gemma-2-2b. 
{However, these methods implement hierarchy through nested prefixes of a single dictionary, so early directions are reused by downstream prefixes, creating a shared-prefix coupling that can limit explicit multi-level concept organization; error in high-level concepts can propagate to low-level concepts (theoretical analysis in~\secref{sec:msae_theory}).} {Another natural approach is to stack multiple SAE layers.}
{However, stacked SAEs compress sparse activation codes through a smaller high-level bottleneck, which can lose information needed for reconstruction {and thus} hinder the learning of lower-level concepts (theoretical analysis in~\secref{sec:ssae_theory}).}


\begin{figure}[t]
  \centering
  \vspace{-0.5cm}
  \includegraphics[width=\linewidth]
  {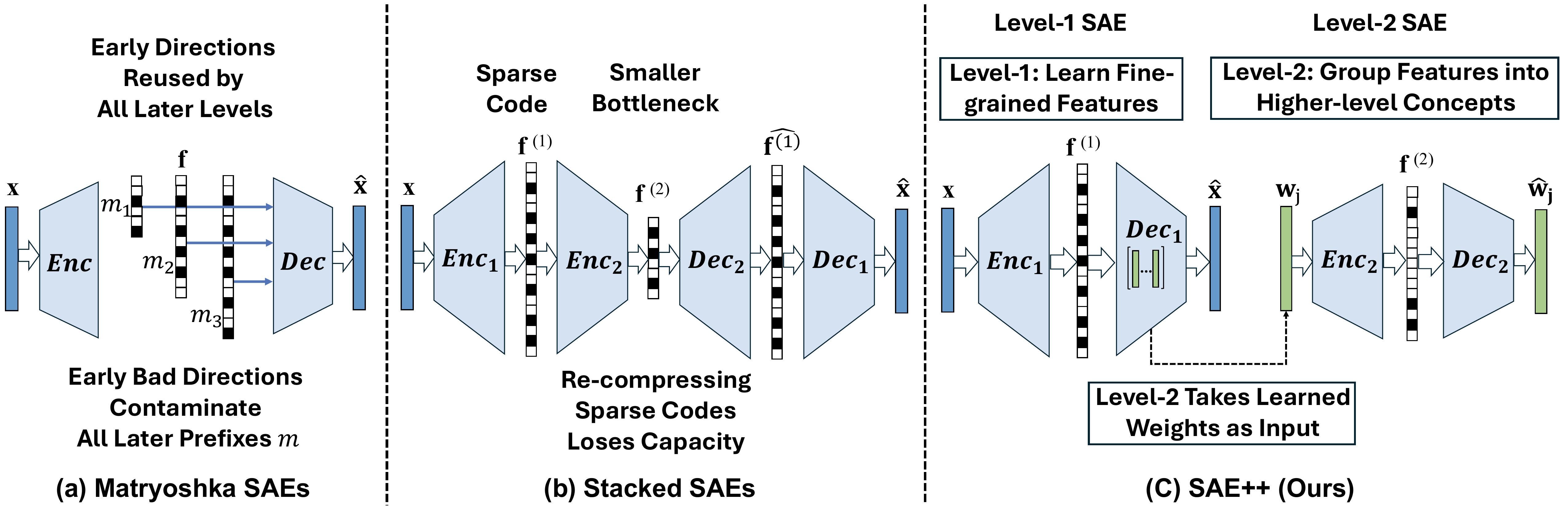}
   \vspace{-0.6cm}
   \caption{Overview of hierarchical design choices for SAEs.
\textbf{(a) Matryoshka SAEs} construct the concept hierarchy through a single nested prefix chain, with early latent directions globally reused across later levels. \textbf{(b) Stacked SAEs} learn the concept hierarchy by re-compressing sparse sample codes through a smaller bottleneck, which introduces an additional capacity constraint. \textbf{(c) Our \modelnames} {takes} a different route: the Level-1 SAE learns low-level features, and the Level-2 SAE is trained directly on the learned Level-1 decoder weights, enabling higher-level abstraction over concepts rather than overly shared prefixes or re-compressed sample codes.}
   \label{fig:overview}
   \vskip -0.5cm
\end{figure}

\figref{fig:overview} shows an overview comparing different designs. 
To address this challenge, this paper makes a systematic attempt to \emph{explore an alternative beyond the nested architecture in Matryoshka SAEs and the stacked SAE architecture}. We propose, to the best of our knowledge, the first cascaded SAE architecture, dubbed \modelnames. Rather than feeding the activations of one SAE layer into another, we jointly train two SAEs: a high-level SAE and a low-level SAE. 
The low-level (Level-1) SAE learns low-level concepts directly from MLLM 
activations. \textbf{Our key insight} is that the MLLM activations, Level-1 decoder 
weight columns, and Level-2 decoder weight columns all lie in the same 
$d$-dimensional activation space. This enables a recursive ``representatives 
of representatives'' construction: the Level-1 decoder atoms represent MLLM 
activations, while the Level-2 SAE treats these learned concept directions as 
its inputs and learns higher-level representatives of them, i.e., 
``concepts of concepts.''
Our contributions are:

\begin{itemize}[nosep]
\item We propose the first cascaded SAE architecture, dubbed \modelnames, that {enables} the learning of hierarchical concepts from MLLMs.
\item {Theoretical analysis shows that naively stacking SAE layers can suffer from a sparse bottleneck failure mode, while Matryoshka-style nested prefixes can amplify local semantic errors through shared-prefix reuse.}
\item {Empirical results across Qwen3-VL, Gemma-3, and LLaVA on multiple visual datasets demonstrate improved hierarchical concept interpretability. Results on activation steering show that the learned concept groups enable effective interventions in MLLM outputs.}
\end{itemize}


\section{Related Work}

\textbf{Multimodal LLMs.} 
{Compared to traditional text-only LLMs, Multimodal LLMs (MLLMs) perceive and reason across modalities such as images, video, audio, and text.}
{Seminal works such as CLIP~\cite{radford2021learningtransferablevisualmodels} aligned visual and textual representations for zero-shot classification, Flamingo~\cite{alayrac2022flamingovisuallanguagemodel} introduced interleaved visual-textual modeling, and BLIP-2~\cite{li2023blip2bootstrappinglanguageimagepretraining} adapted frozen LLMs to visual tasks using lightweight adapters.}
{Modern multimodal systems include closed-source models such as the GPT series~\cite{openai2024gpt4technicalreport}, Gemini series~\cite{geminiteam2025geminifamilyhighlycapable, comanici2025gemini25pushingfrontier, geminiteam2024gemini15unlockingmultimodal}, and Claude~\cite{anthropic2024claude3}, as well as open-source models such as LLaVA~\cite{liu2023visualinstructiontuning}, Qwen-VL~\cite{wang2024qwen2vlenhancingvisionlanguagemodels, bai2025qwen25vltechnicalreport, bai2025qwen3vltechnicalreport}, DeepSeek-VL~\cite{lu2024deepseekvlrealworldvisionlanguageunderstanding}, Pixtral~\cite{agrawal2024pixtral12b}, Gemma~\cite{gemmateam2024gemmaopenmodelsbased, gemmateam2025gemma3technicalreport}, and Llama 3.2 Vision~\cite{meta2024llama32}.}
{While these models show strong performance, it remains unclear how they represent and organize visual concepts internally; this is the focus of our paper.}

\textbf{Sparse Autoencoders.}
Recently, Sparse Autoencoders (SAEs) have emerged as a powerful tool to enhance the mechanistic interpretability of large models.
Generally, SAEs~\cite{bricken2023monosemanticity, cunningham2023sparseautoencodershighlyinterpretable} employ an overcomplete set of basis vectors and a sparsity penalty (typically $\ell_1$) to decompose the dense, polysemantic activations of a neural network into linear combinations of interpretable, monosemantic feature directions.
{Existing variants include foundational $\ell_1$-based SAEs~\cite{bricken2023monosemanticity, karvonen2024measuringprogressdictionarylearning}, Top-K sparsity~\cite{gao2024scalingevaluatingsparseautoencoders}, Gated SAEs~\cite{rajamanoharan2024improvingdictionarylearninggated}, and JumpReLU SAEs~\cite{rajamanoharan2024jumpingaheadimprovingreconstruction}.}
{These architectures improve sparse feature learning through mechanisms such as hard sparsity, learnable gating, and discontinuous activations, but they are primarily designed for flat feature dictionaries.}
Although SAEs facilitate the interpretability of large models, adapting the aforementioned structures to discover \emph{hierarchical} (i.e. multi-level) concepts within multimodal LLMs presents unique challenges. 

Matryoshka SAEs~\cite{bussmann2025learningmultilevelfeaturesmatryoshka, zaigrajew2025interpretingcliphierarchicalsparse} learn hierarchical concepts by
training multiple nested dictionaries of increasing size, with smaller dictionaries encouraged to reconstruct inputs independently, so earlier prefixes learn more general concepts and later prefixes refine them with more specific features.
However, because these hierarchies are implemented through shared nested prefixes, early directions are reused by downstream prefixes, creating a shared-prefix coupling. Consequently, error in high-level concepts can propagate to low-level concepts (see theoretical analysis in~\secref{sec:msae_theory}).
{In contrast, our proposed \modelname{} learns high-level concepts by training a second SAE directly on the decoder weight columns of the low-level SAE, enabling explicit abstraction over learned visual concepts.}

\textbf{Interpreting Multimodal Models with SAEs.}
SAEs have been extensively used to interpret large models.
Early work studied text-only LLMs, showing that SAEs can decompose polysemantic activations into more interpretable features~\cite{bricken2023monosemanticity, cunningham2023sparseautoencodershighlyinterpretable}.
Later work scaled SAEs to larger language models such as GPT-2/GPT-4~\cite{gao2024scalingevaluatingsparseautoencoders}, Claude 3 Sonnet~\cite{templeton2024scaling}, and Gemma 2~\cite{lieberum2024gemmascopeopensparse}.
More recently, concept-based methods have been applied to vision and multimodal models~\cite{PACE,VALC}.
Prior work trains SAEs on CLIP representations to isolate visual concepts~\cite{zaigrajew2025interpretingcliphierarchicalsparse, pach2025sparseautoencoderslearnmonosemantic}.
Other studies analyze multimodal alignment and shared embedding spaces: SAE-V interprets MLLM cross-modal features and uses them for data filtering and alignment improvement~\cite{lou2025saevinterpretingmultimodalmodels}; VL-SAE maps visual and textual representations into a unified concept set for interpreting and enhancing vision-language alignment~\cite{shen2025vlsaeinterpretingenhancingvisionlanguage}; and SAEs trained on VLM embedding spaces reveal sparse linear structures shaped by modality and cross-modal semantic bridges~\cite{papadimitriou2025interpretinglinearstructurevisionlanguage}.
These works demonstrate the utility of SAEs for multimodal interpretability and alignment.
In contrast, our work studies explicit \emph{multi-level visual concept hierarchies} in modern MLLMs, where high-level concepts organize coherent groups of low-level concepts.

\section{Methodology}

\subsection{Preliminary: Sparse Autoencoders (SAEs)}\label{sec:prelim}

Sparse Autoencoders (SAEs) implement a form of sparse dictionary learning designed to decompose the dense, polysemantic activations of large models into interpretable components. Given an input embedding vector $\mathbf{x} \in \mathbb{R}^{d}$ from a specific layer of a (multimodal) LLM, the goal is to learn an overcomplete dictionary of features parameterized by an encoder $\mathbf{W}_{\text{enc}} \in \mathbb{R}^{n \times d}$ and a decoder $\mathbf{W}_{\text{dec}} \in \mathbb{R}^{d \times n}$ (where $n \gg d$). Crucially, the architecture typically employs a shared geometric bias $\mathbf{b} \in \mathbb{R}^{d}$, which is subtracted from the input to center the signal during encoding and added back to the output during decoding.

The encoding process $g: \mathbb{R}^d \to \mathbb{R}^n$ projects the centered input into a sparse latent space, while the decoding process $h: \mathbb{R}^n \to \mathbb{R}^d$ maps the activations back to the input data manifold:
\begin{align}
    \mathbf{f} &= g(\mathbf{x}) = \sigma(\mathbf{W}_{\text{enc}}(\mathbf{x} - \mathbf{b})), \\
    \hat{\mathbf{x}} &= h(\mathbf{f}) = \mathbf{W}_{\text{dec}} \mathbf{f} + \mathbf{b},
\end{align}
where $\sigma(\cdot)$ is a non-linear activation function (e.g., ReLU). The model parameters are learned by minimizing a joint objective $\mathcal{L}$ composed of a reconstruction term $\mathcal{R}$ and a sparsity penalty $\mathcal{S}$ :
\begin{equation}
    \mathcal{L}(\mathbf{x}) = \underbrace{||\mathbf{x} - \hat{\mathbf{x}}||_2^2}_{\mbox{reconstruction }\mathcal{R}(\mathbf{x})} +  \underbrace{\lambda \mathcal{S}\big(g(\x)\big)}_{\mbox{sparsity penalty}},
\end{equation}
{where $\mathcal{S}(\cdot)$ is usually an $\ell_1$ loss}, i.e., $\|\cdot\|_1$, to encourage sparsity. 
While the standard instantiation relies on ReLU activation and an $\ell_1$ sparsity penalty, this framework readily accommodates alternative architectures. For instance, TopK SAEs \cite{gao2024scalingevaluatingsparseautoencoders} substitute the soft $\ell_1$ penalty with a hard sparsity constraint by modifying $\sigma(\cdot)$ to select only the top-$k$ activations. 

{Similarly, Matryoshka SAEs~\cite{bussmann2025learningmultilevelfeaturesmatryoshka} organize the latent space into nested index subsets
$\mathcal{I}_i=\{1,\ldots,m_i\}$ with
$m_1<\cdots<m_L$, enforcing a joint objective
$\sum_{i=1}^L \|\mathbf{x} - (\mathbf{W}_{\text{dec}}^{(:, \mathcal{I}_i)}\mathbf{f}_{\mathcal{I}_i} + \mathbf{b})\|_2^2$
so that smaller prefixes are encouraged to reconstruct the input independently.}

\subsection{From SAE to \modelname}\label{sec:saepp}
Existing SAEs usually treat dictionary atoms as independent vectors, failing to capture the hierarchical correlations inherent in multimodal data (e.g., lower-level concepts ``barrel'' and ``drum'' together form the higher-level concept of ``cylinder''). 

To address this, we propose a new SAE architecture, \modelname, which is a two-level sparse dictionary learning system trained end-to-end. This framework implements hierarchical feature abstraction, systematically organizing atomic concepts (features) from the low-level SAE into high-level semantic structures. Below we introduce \modelname in detail. 
Note that while we describe \modelname as a \emph{two-level} model, our formulation \emph{can be naturally generalized to more than two levels, as shown in~\appref{appsec:generalized_saepp}}. 



Let $\mathbf{x} \in \mathbb{R}^d$ denote an activation vector from a fixed layer of a multimodal LLM. Our \modelname consists of two hierarchical levels trained jointly to capture both fine-grained atomic features and their high-level semantic structures. {\figref{fig:overview}(c) shows an overview of \modelname.}

{\textbf{Level-1 (Low-Level) SAE.} We employ a standard SAE to decompose the input activations $\x$. {This Level-1 SAE has an encoder $g_1$ and a decoder $h_1$, which produce a sparse code $\mathbf{f}^{(1)} \in \mathbb{R}^{n_1}$ and a reconstruction $\hat{\mathbf{x}}$:}}
\begin{align}
    \mathbf{f}^{(1)} &= g_1(\mathbf{x}) = \sigma(\mathbf{W}^{(1)}_{\text{enc}}(\mathbf{x} - \mathbf{b}^{(1)})), \\
    \widehat{\mathbf{x}} &= h_1(\f^{(1)}) = \mathbf{W}_{\text{dec}}^{(1)} \mathbf{f^{(1)}} + \mathbf{b}^{(1)}. \label{eq:recon_l1}
\end{align}
Importantly, here the decoder weight matrix $\mathbf{W}_{\text{dec}}^{(1)}$ contains the atomic concept (feature) directions learned by the model. We denote the columns of this decoder by:
\begin{equation}
    \mathbf{W}_{\text{dec}}^{(1)} = [\mathbf{w}_1, \dots, \mathbf{w}_{n_1}], \quad \mathbf{w}_j \in \mathbb{R}^d, \label{eq:W_dec}
\end{equation}
where each $\mathbf{w}_j$ represents a specific concept direction in the original activation space. 
For each input LLM activation embedding $\x$, we can define the Level-1 loss {as follows:}
\begin{align*}
    \mathcal{L}_1(\mathbf{x}, \tha_1) = ||\mathbf{x} - \hat{\mathbf{x}}||_2^2 + \lambda_1 \SM\big(g_1(\x)\big),
\end{align*}
where $\mathcal{S}(\cdot)$ is usually an $\ell_1$ loss, i.e., $\|\cdot\|_1$, to encourage sparsity. $\tha_1=\{\W_{enc}^{(1)},\W_{dec}^{(1)},\b^{(1)}\}$ denotes parameters in the Level-1 SAE.

{\textbf{Level-2 (High-Level) SAE.} The core innovation of our \modelname{} lies in the structural ``clustering'' performed by the Level-2 (high-level) SAE. Rather than processing features $\x$ or $\f^{(1)}$, our Level-2 SAE operates directly on the atomic concept directions, i.e., the weight columns 
$\{\mathbf{w}_j\}_{j=1}^{n_1}$ of the Level-1 decoder $\mathbf{W}_{\text{dec}}^{(1)}$ (\eqnref{eq:W_dec}). {It uses an encoder-decoder pair to map each Level-1 atom $\w_j$ from the original feature space $\mathbb{R}^d$ to a higher-level latent space $\mathbb{R}^{n_2}$ and reconstruct it back to $\mathbb{R}^d$:}}
\begin{align}
    \mathbf{f}^{(2)}_j &= g_2(\w_j) = \sigma(\mathbf{W}^{(2)}_{\text{enc}}(\w_j - \mathbf{b}^{(2)})), \\
    \widehat{\w}_j &= h_2(\mathbf{f}^{(2)}_j) = \mathbf{W}_{\text{dec}}^{(2)} \mathbf{f}^{(2)}_j + \mathbf{b}^{(2)}.
\end{align}
{\underline{\emph{Our key insight}} is that the MLLM activations and decoder columns at both levels lie in the same space $\mathbb{R}^d$. Level 2 therefore learns representatives of the Level-1 concept directions, i.e., ``concepts of concepts'' in the original activation space.}
This Level-2 (high-level) SAE is learned with the objective:
\begin{align}
    \mathcal{L}_2(\mathbf{W}_{\text{dec}}^{(1)}, \tha_2) 
    =
    \frac{1}{n_1} \sum_{j=1}^{n_1} \Big(
    \|\mathbf{w}_j - h_2(g_2(\mathbf{w}_j))\|_2^2 
    + \lambda_2 \mathcal{S}(g_2(\mathbf{w}_j))\Big),
\label{eq:L2}
\end{align}
where $\mathcal{S}(\cdot)$ is usually an $\ell_1$ loss, i.e., $\|\cdot\|_1$, to encourage sparsity. 
$\tha_2=\{\W_{enc}^{(2)},\W_{dec}^{(2)},\b^{(2)}\}$ denotes parameters in the Level-2 SAE.

{\textbf{Joint Optimization.} {The final objective combines the losses of the Level-1 SAE ($\mathcal{L}_1$) and the Level-2 SAE ($\mathcal{L}_2$):}}
\begin{equation}
    \mathcal{L}_{\text{final}}(\tha_1, \tha_2) = \mathbb{E}_{\mathbf{x}}[\mathcal{L}_1(\mathbf{x}, \tha_1)] + \alpha \mathcal{L}_2 (\mathbf{W}_{\text{dec}}^{(1)}, \tha_2 ).
\end{equation}
{Here, $\alpha$ is a hyperparameter that balances low-level concept learning and the high-level concept abstraction.}
Note that $\mathbf{W}_{\text{dec}}^{(1)}$ is part of the Level-1 SAE's parameters $\tha_1$. 
{Empirical training diagnostics for the joint objective are provided in~\appref{appsec:empirical_convergence}.}

\textbf{Dynamic Masking of Dead Latents.}
For computational efficiency, we apply the Level-2 SAE only to active Level-1 decoder atoms (weight columns) $\w_j$ in the current mini-batch.
At training step \(t\), given mini-batch \(\mathcal{B}\), we define
\(\mathcal{A}_t=\{\mathbf{w}_j~|~\sum_{\mathbf{x}\in\mathcal{B}}\mathbf{1}(|[g_1(\mathbf{x})]_j|>\epsilon_0)\ge1\}\),
where \(\mathbf{w}_j\) is the \(j\)-th Level-1 decoder atom, \([g_1(\mathbf{x})]_j\) is its Level-1 activation on input \(\mathbf{x}\), and \(\epsilon_0>0\) is a small threshold.
We compute the Level-2 loss only over \(\mathcal{A}_t\).
{This} masking affects only which Level-1 atoms are included in the Level-2 objective, 
not the architecture or final objective (more details in~\appref{appsec:dynamic_masking}).

\subsection{{Discussion and Insight: SAE++ Learns Concepts of Concepts}}
\label{sec:concepts_of_concepts}

\textbf{Level-1 Decoder's Weight Columns $\w_j$ Are Representative Points in the MLLM Activation Space.} Our key observation is that the Level-1 decoder weight columns and the original MLLM activations lie in the same space and have the same dimensionality $d$. Specifically, given an MLLM activation $\x \in \mathbb{R}^d$, $\W_{dec}^{(1)} \in \mathbb{R}^{d \times n_1}$ has $d$ rows and $n_1$ columns, and each column $\w_j \in \mathbb{R}^d$ also has $d$ dimensions. \eqnref{eq:recon_l1}, $h_1(f^{(1)})=\W_{dec}^{(1)}f^{(1)}+\b^{(1)}$, reconstructs each MLLM activation as a sparse linear combination of these columns. Thus, although $\W_{dec}^{(1)}$ is formally a parameter matrix, its columns are learned $d$-dimensional ``representative points'' (i.e., concepts) of the $d$-dimensional MLLM activation embeddings, rather than arbitrary model weights.

\textbf{Level-2 Decoder's Weight Columns Are ``Representatives of Representatives''.} This directly motivates our Level-2 SAE. If the Level-1 SAE learns ``representative'' concepts from MLLM activations, Level 2 learns ``representatives of these representatives'', i.e., ``concepts of concepts.'' Specifically, it maps each $\w_j \in \mathbb{R}^d$ to a sparse code $f_j^{(2)}=g_2(\w_j)$ in $\mathbb{R}^{n_2}$ and reconstructs it using $h_2(f_j^{(2)})=\W_{dec}^{(2)}f_j^{(2)}+\b^{(2)}$. Therefore, Level-1 concepts such as ``Truck'' and ``Jeep'' can both be mapped to the same higher-level (Level-2) factor corresponding to ``Vehicle''. Note that the Level-2 decoder weight columns and the Level-1 decoder weight columns are also in the same space and have the same dimensionality $d$.

To summarize, (1) the Level-2 decoder weight columns, (2) the Level-1 decoder weight columns, and (3) the original MLLM activations lie in the same space and have the same dimensionality $d$. Therefore, our \modelname{}'s cascaded architecture ensures that 
\begin{itemize}[nosep]
    \item the Level-1 decoder weight columns are representative points (i.e., Level-1 concepts) of the MLLM activations and that
    \item the Level-2 decoder weight columns are representative points (i.e., ``representatives of these representatives'' or ``concepts of concepts'') of the Level-1 decoder weight columns. 
\end{itemize}

\section{Theoretical Analysis}
\label{sec:theory}

In this section, we provide theoretical analysis on why two natural hierarchical SAE designs, Stacked SAEs and Matryoshka SAEs, may fall short in learning multi-level concepts. 
Specifically, we show that these two designs may fail for different reasons: 
stacked SAEs face a sparse-code bottleneck mismatch, while Matryoshka SAEs can reuse early semantic errors across downstream prefixes. 
\emph{The purpose of this section is not to prove that these baselines are always ineffective, but to isolate structural failure modes that our \modelnames{} {avoids} by construction.}
{These results identify the structural effects of sparse-code bottlenecks and nested-prefix reuse, providing a theoretical rationale for \modelname{}'s decoder-atom cascade.}
\textbf{All proofs are in~\appref{appsec:theory}.}

\subsection{Why Stacked SAEs Fail}\label{sec:ssae_theory}

\textbf{Failure of Stacked SAEs.} 
Consider a stacked SAE with the architecture (similar to~\figref{fig:overview}(b))
\begin{align}
d\rightarrow n \rightarrow m\rightarrow n\rightarrow d,
\label{eq:arch}
\end{align}
where the first SAE (the encoder-decoder pair $d\rightarrow n$ and $n\rightarrow d$) maps an activation to a sparse code
\[
\mathbf{z}\in
\mathcal{S}_{n,k_1}(B)
:=
\{\mathbf{z}\in\mathbb{R}^n:\|\mathbf{z}\|_0\le k_1,\ \|\mathbf{z}\|_2\le B\}.
\]
The second SAE uses \(g:\mathbb{R}^n\to\mathbb{R}^m\) and \(h:\mathbb{R}^m\to\mathbb{R}^n\) to reconstruct \(\mathbf{z}\).
If its bottleneck is sparser,
\[
g(\mathbf{z})\in
\mathcal{S}_{m,k_2}(R)
:=
\{\mathbf{u}\in\mathbb{R}^m:\|\mathbf{u}\|_0\le k_2,\ \|\mathbf{u}\|_2\le R\},
\qquad
k_2<k_1,
\]
then the second SAE must preserve a \(k_1\)-sparse code space using a \(k_2\)-sparse bottleneck.
This is exactly the setting one might hope would learn higher-level abstractions: a lower-dimensional and sparser code should represent more compressed concepts.
The theorem below shows that this intuition conflicts with uniform reconstruction.

\begin{theorem}[\textbf{Failure of Stacked SAEs with a Sparser Bottleneck}]
\label{thm:stacked_failure_sparse_main}
Assume uniform reconstruction
\[
\sup_{\mathbf{z}\in\mathcal{S}_{n,k_1}(B)}
\|h(g(\mathbf{z}))-\mathbf{z}\|_2 \le \varepsilon,
\]
where \(h\) is \(L_h\)-Lipschitz, and assume \(g(\mathbf{z})\in\mathcal{S}_{m,k_2}(R)\) with \(k_2<k_1\).
{If \(\varepsilon\le B/8\) and the support-packing capacity inequality in Appendix~\ref{appsec:stacked_full_proof}, Eq.~\eqref{eq:stacked_capacity_condition_app}, holds, no such encoder-decoder pair can uniformly reconstruct \(\mathcal{S}_{n,k_1}(B)\). For the empirically used \(k_1=6\), \(k_2=2\), this inequality holds for sufficiently large \(n\) when \(m<n\) and \(RL_h/B\) remains bounded.}
\end{theorem}

Theorem~\ref{thm:stacked_failure_sparse_main} shows why naive stacking is not an effective architecture:
it compresses the combinatorial space of sparse sample codes through a smaller sparse bottleneck.
The issue is not caused by a particular optimizer or activation function; it follows from the geometry of sparse code spaces.

\textbf{\modelnames {Avoids} Similar Failure.} 
In contrast, \modelnames{} {avoids} this issue because its Level-2 SAE is trained on Level-1 decoder atoms $\w_j$,
\[
\mathcal{L}_2
=
\frac{1}{n_1}\sum_{j=1}^{n_1}
\left(
\|\mathbf{w}_j-h_2(g_2(\mathbf{w}_j))\|_2^2
+
\lambda_2\mathcal{S}(g_2(\mathbf{w}_j))
\right),
\]
where \(\mathbf{w}_j\in\mathbb{R}^d\) is a Level-1 decoder weight column.
Thus, the second level abstracts over concept directions rather than sparse sample codes (more details and theoretical analysis in~\appref{appsec:theory}). 

\subsection{Shared-Prefix Error Amplification in Matryoshka SAEs}
\label{sec:msae_theory}

\textbf{Shared-Prefix Reuse of Matryoshka SAEs.}  
Matryoshka SAEs build hierarchy by ordering decoder columns \(\mathbf{d}_1,\mathbf{d}_2,\dots\), each a feature \textbf{direction} in activation space (like $\w_1,\w_2,\dots$ in \modelname{}). The \(t\)-th prefix uses the first \(m_t\) directions, and all later prefixes reuse them. Hence, early directions are shared across levels. If such a direction captures reconstruction-relevant but \emph{semantically irrelevant} variation, its error propagates across multiple levels. We refer to such directions as \textbf{nuisance directions}.

\textbf{Formalizing Prefix Reuse.}  
Matryoshka SAEs optimize nested index sets
$
\mathcal{I}_t=\{1,\ldots,m_t\},~ m_1<\cdots<m_L,
$ 
with the objective function
$
\sum_{t=1}^{L}
\left\|
\mathbf{x}-
\left(
\mathbf{W}_{\mathrm{dec}}^{(:,\mathcal{I}_t)}
\mathbf{f}_{\mathcal{I}_t}
+\mathbf{b}
\right)
\right\|_2^2
$, where $\mathbf{W}_{\mathrm{dec}}^{(:,\mathcal{I}_t)}$ and $\mathbf{f}_{\mathcal{I}_t}$ are subsets of $\mathbf{W}_{\mathrm{dec}}$ and $\mathbf{f}$ indexed by $\mathcal{I}_t$.
Because prefix \(t\) can only use the first \(m_t\) columns, any column at position \(p\) is reused in all prefixes with \(m_t \ge p\). We define the reuse count
$
C_{\mathrm{MSAE}}(p):=|\{t:m_t\ge p\}|
$, 
which quantifies how often a direction is reused; earlier columns have larger reuse.

\textbf{From Prefix Reuse to Semantic Mismatch.}  
We analyze semantics via an orthonormal linear surrogate. Let
$
D=[\mathbf{d}_1,\ldots,\mathbf{d}_{m_L}],~
U_t(D)=\mathrm{span}\{\mathbf{d}_1,\ldots,\mathbf{d}_{m_t}\}.
$
With projector \(P_U\), prefix \(t\) reconstructs by projection, giving the idealized reconstruction loss
$
\mathbb{E}_{\mathbf{x}}\|\mathbf{x}-P_{U_t(D)}\mathbf{x}\|_2^2
$.
This linear assumption is used only for the shared-prefix analysis; details are in~\appref{app:shared_prefix_proofs}. 

To measure semantic correctness, we compare \(U_t(D)\) with an ideal semantic subspace \(\mathcal{M}_t\) (see \appref{app:shared_prefix_proofs} for the derivation of the  projector-distance form):
\[
\mathcal{E}_t^{\mathrm{MSAE}}(D)=\|P_{U_t(D)}-P_{\mathcal{M}_t}\|_F^2,\quad
\mathcal{E}_{\mathrm{MSAE}}^{\mathrm{tot}}(D)=\sum\nolimits_{t=1}^L \mathcal{E}_t^{\mathrm{MSAE}}(D).
\]
Thus, mismatch reflects deviation from intended semantic structure.

\textbf{Normalized Semantic Mismatch.}  
Because early directions appear in many prefixes, their errors are counted repeatedly in \(\mathcal{E}_{\mathrm{MSAE}}^{\mathrm{tot}}\). To isolate local effects, we define a normalized error counting each nuisance direction once.  

Let nuisance directions be \(\mathbf{d}_{p_i}=\mathbf{n}_i\), \(i=1,\dots,q\), each orthogonal to all \(\mathcal{M}_t\). For prefix \(t\) with \(m_t\ge p_i\), define a repaired subspace \(\widetilde{U}_t^{(-i)}(D)\) replacing \(\mathbf{n}_i\). The local error reduction is
\[
e_{i,t}^{\mathrm{MSAE}}
=
\mathcal{E}_t^{\mathrm{MSAE}}(D)
-
\|P_{\widetilde{U}_t^{(-i)}(D)}-P_{\mathcal{M}_t}\|_F^2.
\]
Averaging over reused prefixes yields
\[
\overline{\mathcal{E}}_{\mathrm{MSAE}}
=
\sum\nolimits_{i=1}^q
\frac{1}{C_{\mathrm{MSAE}}(p_i)}
\sum\nolimits_{t:m_t\ge p_i} e_{i,t}^{\mathrm{MSAE}}.
\]
This removes duplication effects and enables \emph{fair comparison between Matryoshka SAEs and \modelnames{}}.

\textbf{Semantic Mismatch for \modelnames{}.}  
In contrast, \modelnames{} {does} not reuse directions across levels. Each Level-1 atom \(\mathbf{w}_j\) has a single parent assignment. Let \(c(j)\) be its ideal parent, with subspace \(\mathcal{P}_{c(j)}\), and \(\widehat{\mathcal{P}}_j\) the assigned one. The local error is
\[
e_j^{\mathrm{\modelname{}}}
=
\|P_{\widehat{\mathcal{P}}_j}-P_{\mathcal{P}_{c(j)}}\|_F^2,
\]
and for a set \(\mathcal{B}\),
\[
\overline{\mathcal{E}}_{\mathrm{\modelname{}}}(\mathcal{B})
=
\sum\nolimits_{j\in\mathcal{B}} e_j^{\mathrm{\modelname{}}}.
\]
Since each error is counted once in \modelname{}, no reuse normalization is needed.

\begin{theorem}[\textbf{Prefix Reuse Amplifies Matryoshka Semantic Error}]
\label{thm:shared_prefix_main}
Assume \(q\) mutually orthonormal nuisance directions
\(\mathbf{n}_1,\dots,\mathbf{n}_q\) are orthogonal to every target semantic subspace \(\mathcal{M}_t\), and occupy positions \(p_1,\dots,p_q\) in \(D\).
Then
$
\mathcal{E}^{\mathrm{tot}}_{\mathrm{MSAE}}(D)
\ge
2\sum_{i=1}^{q}C_{\mathrm{MSAE}}(p_i)
$. 
Moreover, for any \(\mathcal{B}\) with \(|\mathcal{B}|\le q\), we have
\begin{align}
\overline{\mathcal{E}}_{\mathrm{\modelname{}}}(\mathcal{B})
\le
2q
\le
\overline{\mathcal{E}}_{\mathrm{MSAE}}.
\end{align}
\end{theorem}


{\textbf{\modelnames{} {Keeps} Semantic Errors Local.} Theorem~\ref{thm:shared_prefix_main} shows that the same number of local semantic errors can lead to a larger error in Matryoshka SAEs than in \modelnames{}.
In Matryoshka SAEs, an erroneous early decoder direction is reused by multiple prefixes, so its error is 
propagated to lower-level (children) concepts. 
In \modelnames{}, each erroneous atom assignment is counted once and does not propagate to lower-level concepts. Thus, \modelnames{} {avoids} shared-prefix amplification and {yields} \emph{smaller or identical semantic error} compared to Matryoshka SAEs under the same local error budget.}

\begin{table*}[t]
\centering
\vskip -0.4cm
\caption{
\textbf{HMS}$_{\mathrm{mean}}$ for different methods.
More HMS results are in~\appref{appsec:full_hms}. 
We mark the best results in \textbf{bold} and the second best with \underline{underline}.
}
\label{tab:hms_mean}
\vskip 0.1cm
\scriptsize
\setlength{\tabcolsep}{2.6pt}
\renewcommand{\arraystretch}{0.82}
\resizebox{0.96\linewidth}{!}{%
\begin{tabular}{llccccccccc}
\toprule
\textbf{MLLM} & \textbf{Data}
& \textbf{BTK}
& \textbf{TopK}
& \textbf{ReLU}
& \textbf{P-Ann.}
& \textbf{GSAE}
& \textbf{JReLU}
& \textbf{MSAE}
& \textbf{Stack.}
& \textbf{\modelname{}} \\
\midrule
\multirow{4}{*}{Qwen3}
& Color       & 0.421 & 0.122 & 0.214 & 0.545 & 0.655 & 0.147 & \underline{0.978} & 0.601 & \textbf{0.983} \\
& ImageNet    & 0.256 & 0.155 & 0.562 & \underline{0.639} & 0.561 & 0.076 & 0.584 & 0.332 & \textbf{0.778} \\
& COCO        & 0.127 & 0.151 & 0.328 & 0.126 & 0.281 & 0.121 & \underline{0.461} & 0.144 & \textbf{0.980} \\
& iNat.       & 0.318 & 0.325 & 0.193 & 0.416 & 0.261 & 0.102 & \underline{0.588} & 0.231 & \textbf{0.741} \\
\midrule
\multirow{4}{*}{Gemma-3}
& Color       & 0.667 & 0.445 & 0.222 & 0.660 & 0.566 & 0.160 & \underline{0.985} & 0.710 & \textbf{0.993} \\
& ImageNet    & 0.161 & 0.134 & 0.080 & 0.150 & 0.256 & 0.157 & \underline{0.687} & 0.488 & \textbf{0.775} \\
& COCO        & 0.280 & 0.053 & 0.466 & 0.160 & 0.666 & 0.167 & \underline{0.702} & 0.412 & \textbf{0.891} \\
& iNat.       & 0.503 & 0.230 & 0.089 & 0.148 & 0.225 & 0.415 & \underline{0.869} & 0.460 & \textbf{0.913} \\
\midrule
\multirow{4}{*}{LLaVA}
& Color       & 0.331 & 0.242 & 0.317 & 0.336 & 0.432 & 0.434 & \underline{0.962} & 0.577 & \textbf{0.971} \\
& ImageNet    & 0.191 & 0.246 & 0.110 & 0.493 & 0.195 & 0.410 & \underline{0.747} & 0.230 & \textbf{0.896} \\
& COCO        & 0.397 & 0.396 & 0.321 & 0.334 & 0.090 & 0.237 & \underline{0.811} & 0.284 & \textbf{0.941} \\
& iNat.       & 0.333 & 0.316 & 0.364 & 0.266 & 0.363 & 0.258 & \underline{0.759} & 0.260 & \textbf{0.780} \\
\bottomrule
\end{tabular}%
}
\renewcommand{\arraystretch}{1.0}
\vskip -0.4cm
\end{table*}

\begin{table*}[t]
\centering
\vskip -0.1cm
\caption{
\textbf{HMS}$_{\mathrm{med}}$ for different methods.
More HMS results are in~\appref{appsec:full_hms}. 
We mark the best results in \textbf{bold} and the second best with \underline{underline}.
}
\label{tab:hms_median}
\vskip 0.1cm
\scriptsize
\setlength{\tabcolsep}{2.6pt}
\renewcommand{\arraystretch}{0.82}
\resizebox{0.96\linewidth}{!}{%
\begin{tabular}{llccccccccc}
\toprule
\textbf{MLLM} & \textbf{Data}
& \textbf{BTK}
& \textbf{TopK}
& \textbf{ReLU}
& \textbf{P-Ann.}
& \textbf{GSAE}
& \textbf{JReLU}
& \textbf{MSAE}
& \textbf{Stack.}
& \textbf{\modelname{}} \\
\midrule
\multirow{4}{*}{Qwen3}
& Color       & 0.420 & 0.108 & 0.213 & 0.545 & 0.520 & 0.127 & \underline{0.999} & 0.628 & \textbf{0.999} \\
& ImageNet    & 0.141 & 0.139 & 0.604 & \underline{0.798} & 0.561 & 0.068 & 0.615 & 0.337 & \textbf{1.000} \\
& COCO        & 0.122 & 0.144 & 0.126 & 0.126 & 0.192 & 0.117 & \underline{0.460} & 0.147 & \textbf{1.000} \\
& iNaturalist       & 0.244 & 0.265 & 0.200 & 0.140 & 0.193 & 0.100 & \underline{0.598} & 0.225 & \textbf{0.974} \\
\midrule
\multirow{4}{*}{Gemma-3}
& Color       & 0.641 & 0.389 & 0.224 & 0.658 & 0.612 & 0.130 & \underline{0.985} & 0.661 & \textbf{0.994} \\
& ImageNet    & 0.159 & 0.135 & 0.080 & 0.090 & 0.088 & 0.087 & \underline{0.753} & 0.499 & \textbf{0.838} \\
& COCO        & 0.280 & 0.053 & 0.466 & 0.160 & \underline{0.828} & 0.174 & 0.725 & 0.508 & \textbf{1.000} \\
& iNaturalist       & 0.503 & 0.230 & 0.089 & 0.148 & 0.096 & 0.144 & \underline{0.886} & 0.456 & \textbf{0.961} \\
\midrule
\multirow{4}{*}{LLaVA}
& Color       & 0.322 & 0.240 & 0.248 & 0.325 & 0.353 & 0.309 & \underline{0.938} & 0.679 & \textbf{0.948} \\
& ImageNet    & 0.175 & 0.218 & 0.081 & 0.498 & 0.198 & 0.276 & \underline{0.779} & 0.199 & \textbf{0.965} \\
& COCO        & 0.330 & 0.338 & 0.254 & 0.271 & 0.090 & 0.207 & \underline{0.865} & 0.203 & \textbf{0.997} \\
& iNaturalist       & 0.298 & 0.287 & 0.312 & 0.242 & 0.293 & 0.212 & \underline{0.786} & 0.229 & \textbf{0.987} \\
\bottomrule
\end{tabular}%
}
\renewcommand{\arraystretch}{1.0}
\vskip -0.35cm
\end{table*}

\section{Experiments}
\label{sec:experiments}

We evaluate \modelname{} against multiple SAE baselines on three MLLM models and four visual datasets. 

\subsection{Experimental Setup}

\textbf{MLLMs and Datasets.}
We evaluate our method using \texttt{Qwen3-VL-4B-Instruct}, \texttt{Gemma-3-4B-IT}, and \texttt{LLaVA-1.5-13B}.
Activations are extracted from Qwen3-VL visual block 23, Gemma-3 vision tower layer 26, and LLaVA language backbone layer 39, respectively, using the prompt ``Describe the image content accurately.''
We use Color~\cite{wang2024probabilisticconceptualexplainerstrustworthy}, ImageNet~\cite{5206848}, iNaturalist~\cite{vanhorn2021benchmarkingrepresentationlearningnatural}, and COCO~\cite{lin2015microsoftcococommonobjects}.
Dataset sampling and implementation details are provided in~\appref{appsec:exp_setup}.

\textbf{Baselines.}
We compare with BatchTopK (\textbf{BTK})~\cite{bussmann2024batchtopk}, \textbf{TopK} SAE~\cite{gao2024scalingevaluatingsparseautoencoders}, \textbf{ReLU} SAE~\cite{bricken2023monosemanticity}, P-Annealing (\textbf{P-Ann.})~\cite{karvonen2024measuringprogressdictionarylearning}, Gated SAE (\textbf{GSAE})~\cite{rajamanoharan2024improvingdictionarylearninggated}, JumpReLU (\textbf{JReLU})~\cite{rajamanoharan2024jumpingaheadimprovingreconstruction}, Matryoshka SAE (\textbf{MSAE})~\cite{bussmann2025learningmultilevelfeaturesmatryoshka, zaigrajew2025interpretingcliphierarchicalsparse}, and Stacked SAE (\textbf{Stack.}).
More baseline details are in~\appref{appsec:baselines}.

\textbf{Metrics.}
We evaluate multi-level semantic coherence with Hierarchical Mono-Semanticity (HMS), adapted from the mono-semanticity score~\cite{pach2025sparseautoencoderslearnmonosemantic}.
{For each Level-1 concept \(j\), we compute a semantic vector \(\mathbf{R}_j\) as an activation-weighted average of fixed-backbone image embeddings, using the normalized activations described in~\appref{appsec:hms}.} 
Let \(\pi(j)\) denote the Level-2 parent of Level-1 concept \(j\), and let \(\mathcal{C}_k=\{j:\pi(j)=k\}\) be the children of Level-2 concept \(k\).
Similar to~\cite{pach2025sparseautoencoderslearnmonosemantic}, we define
\[
\mathrm{HMS}(k)
=
\frac{2}{m(m-1)}
\sum\nolimits_{\substack{u,v\in\mathcal{C}_k\\u<v}}
\frac{\mathbf{R}_u^\top\mathbf{R}_v}
{\|\mathbf{R}_u\|_2\|\mathbf{R}_v\|_2},
\qquad
m=|\mathcal{C}_k|.
\]
HMS measures whether Level-1 concepts under the same Level-2 parent are semantically coherent.
{Cluster sizes, atom coverage, a shared parent-assignment control, and matched-sparsity reconstruction results are provided in~\appref{appsec:additional_evaluation}.}
{In the main paper, we report
{\(\mathrm{HMS}_{\mathrm{mean}}=\frac{1}{|\mathcal{K}|}\sum_{k\in\mathcal{K}}\mathrm{HMS}(k)\)
and
\(\mathrm{HMS}_{\mathrm{med}}=\operatorname{median}_{k\in\mathcal{K}}\mathrm{HMS}(k)\)
over eligible parents \(\mathcal{K}=\{k:|\mathcal{C}_k|\ge2\}\),}}
Results for more metrics, $\mathrm{HMS}_{\min}$ and $\mathrm{HMS}_{\max}$,
are in~\appref{appsec:hms} and~\ref{appsec:full_hms}. 
{For concept steering, we use {four LLM judges} to evaluate whether the target concept appears after insertion or is removed after suppression.}
{More details are provided in~\appref{appsec:multi_judge}.}

\begin{figure*}[t]
  \centering
  \includegraphics[width=\linewidth]{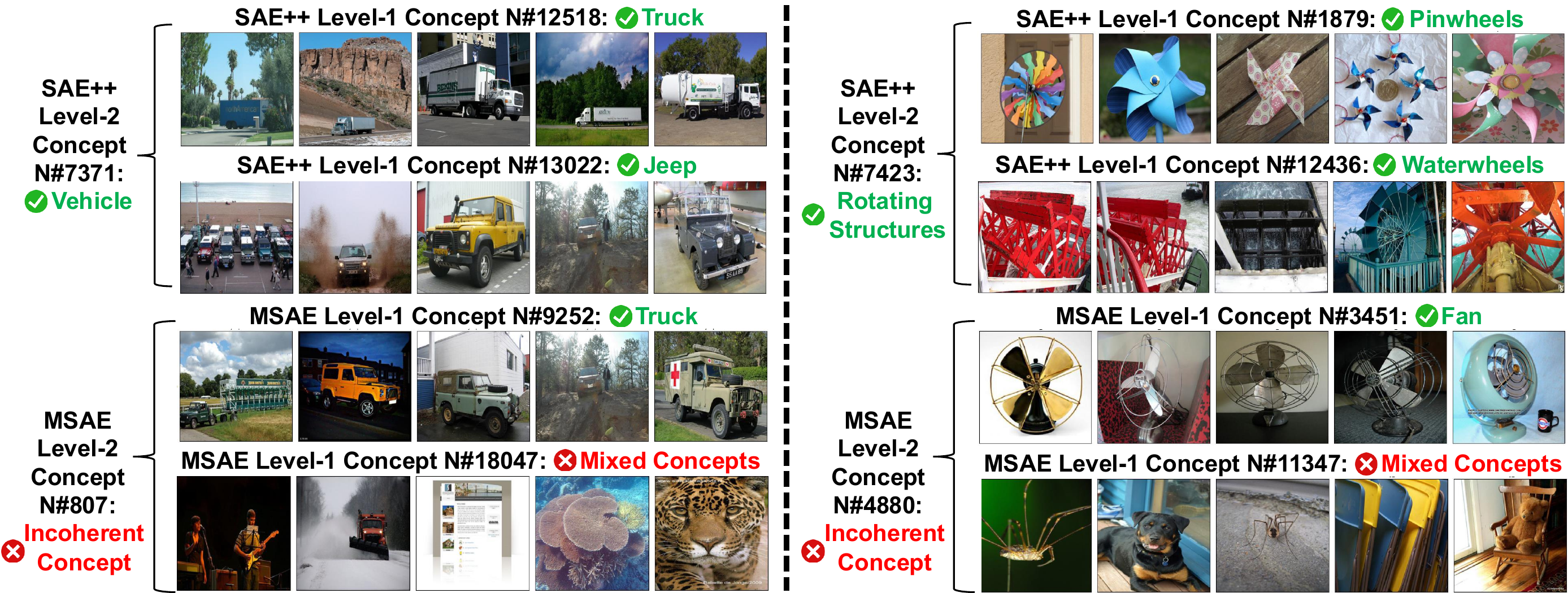}
   \vskip -0.15cm
   \caption{Qualitative examples of multi-level (Level-1 and Level-2) concepts discovered by \modelname{} and Matryoshka SAEs (MSAEs). \textbf{Left:} \modelname{} groups visually coherent Level-1 concepts, \emph{Truck} and \emph{Jeep}, into the same Level-2 concept, \emph{Vehicle}. In contrast, the matched MSAE concepts are less semantically consistent and often mix weaker or unrelated visual patterns. \textbf{Right:} \modelname{} groups visually coherent Level-1 concepts, \emph{Pinwheels} and \emph{Waterwheels}, into the same Level-2 concept, \emph{Rotating Structures}, while the two Level-1 concepts under the same MSAE Level-2 concept are less semantically consistent.
Additional qualitative results are provided in \appref{appsec:more_qualitative}. 
}\label{fig:qualitative}
   \vskip -0.05cm
\end{figure*}



\subsection{Results}
\label{sec:results}

\textbf{Quantitative Results.}
\tabref{tab:hms_mean} and \tabref{tab:hms_median} show HMS$_{\mathrm{mean}}$ and HMS$_{\mathrm{med}}$ across three MLLMs and four datasets.
\modelname{} achieves the best HMS$_{\mathrm{mean}}$ and HMS$_{\mathrm{med}}$ in all settings, with especially large gains on COCO and ImageNet, where learning high-level concepts is more challenging.
Matryoshka SAE is usually the strongest baseline, but it consistently trails \modelname{}, suggesting that explicit concept abstraction yields more coherent high-level concepts than shared-prefix hierarchies.
Complete HMS results, including HMS$_{\min}$, HMS$_{\mathrm{med}}$, HMS$_{\max}$ , and HMS$_{\mathrm{mean}}$, are provided in~\appref{appsec:full_hms}.

\textbf{Qualitative Results.}
\figref{fig:qualitative} visualizes representative multi-level concepts discovered by \modelname{} and Matryoshka SAE.
\modelname{} groups fine-grained but related Level-1 concepts under coherent \mbox{Level-2} abstractions, such as \emph{Truck}/\emph{Jeep} under a \emph{Vehicle} concept and \emph{Pinwheels}/\emph{Waterwheels} under a \emph{Rotating Structure} concept. 
Matched Matryoshka concepts are related but often mix less coherent visual patterns.
These examples support the HMS results: \modelname{} preserves visually specific Level-1 concepts while organizing them into coherent high-level groups.
Additional examples are provided in~\appref{appsec:more_qualitative}.

\begin{table*}[t]
\centering
\begingroup
\color{black}
\caption{
Concept steering success rates (\%) on \texttt{Qwen3-VL-4B-Instruct} under four LLM judges.
We report success rates on 1,000 COCO evaluations from ten Level-2 concepts; each judge scores the same generated outputs.
Judge Avg. is the unweighted mean across the four judges.
Higher is better.
}
\label{tab:steering}
\small
\setlength{\tabcolsep}{4pt}
\renewcommand{\arraystretch}{1.05}
\begin{tabular}{lrrrrr@{\hspace{12pt}}rrrrr}
\toprule
& \multicolumn{5}{c}{\textbf{Insertion (\%)}}
& \multicolumn{5}{c}{\textbf{Suppression (\%)}} \\
\cmidrule(lr){2-6}\cmidrule(lr){7-11}
\textbf{Judge}
& \textbf{NS} & \textbf{Rand.} & \textbf{MSAE} & \textbf{Single} & \textbf{\modelname{}}
& \textbf{NS} & \textbf{Rand.} & \textbf{MSAE} & \textbf{Single} & \textbf{\modelname{}} \\
\midrule
Gemini-2.5-Flash & 10.6 & 10.8 & 40.7 & 58.4 & \textbf{67.8} & 0.0 & 12.4 & 25.2 & 35.6 & \textbf{40.5} \\
Gemini-3-Flash   & 9.8  & 12.1 & 41.3 & 57.5 & \textbf{70.4} & 0.0 & 14.5 & 21.9 & 34.2 & \textbf{43.6} \\
Nemotron-30B     & 17.1 & 16.4 & 40.7 & 62.7 & \textbf{73.3} & 0.0 & 13.2 & 24.5 & 36.3 & \textbf{40.6} \\
Opus-4.8        & 4.7  & 5.2  & 41.0 & 60.9 & \textbf{72.6} & 0.0 & 9.6  & 22.6 & 30.3 & \textbf{35.7} \\
\midrule
Judge Avg.      & 10.6 & 11.1 & 40.9 & 59.9 & \textbf{71.0} & 0.0 & 12.4 & 23.6 & 34.1 & \textbf{40.1} \\
\bottomrule
\end{tabular}
\endgroup
\end{table*}

\textbf{Concept Steering.}
Beyond concept discovery, we evaluate whether \modelname{} supports group-level interventions.
On COCO test images, we randomly select 10 Level-2 concepts and 100 images per concept, yielding 1,000 evaluations.
We steer all SAE Level-1 concepts under the same Level-2 concept and feed the reconstructed activations back into the MLLM residual stream.
The steering scale \(\textbf{u}_s\) is computed from the mean activation of the top-10 training images for each cluster; we use \(+3\textbf{u}_s\) for concept insertion and \(-3\textbf{u}_s\) for concept suppression.
We compare against no steering (\textbf{NS}), {random concepts} (\textbf{Rand.}), matched Matryoshka SAE clusters (\textbf{MSAE}), and best individual \modelname{} Level-1 concept (\textbf{Single}).
{We evaluate the same generated outputs using four LLM judges.
As shown in~\tabref{tab:steering}, \modelname{} achieves the highest insertion and suppression success rates under every judge, with four-judge averages of 71.0\% and 40.1\%, respectively.}

\figref{fig:steer} shows a case study on the ``\emph{Table}'' concept. 
We compare unsteered generation, matched Matryoshka SAE steering, the best individual \modelname{} Level-1 neuron, and full \modelname{} Level-2 clustered steering.
  In both insertion and suppression, \modelname{}'s steering most effectively controls the generated response in terms of the ``\emph{Table}'' concept. 
More details are provided in~\appref{app:steering}.

\begin{figure*}[!t]
  \centering
  \vspace{0.0cm}
  \includegraphics[width=\linewidth]{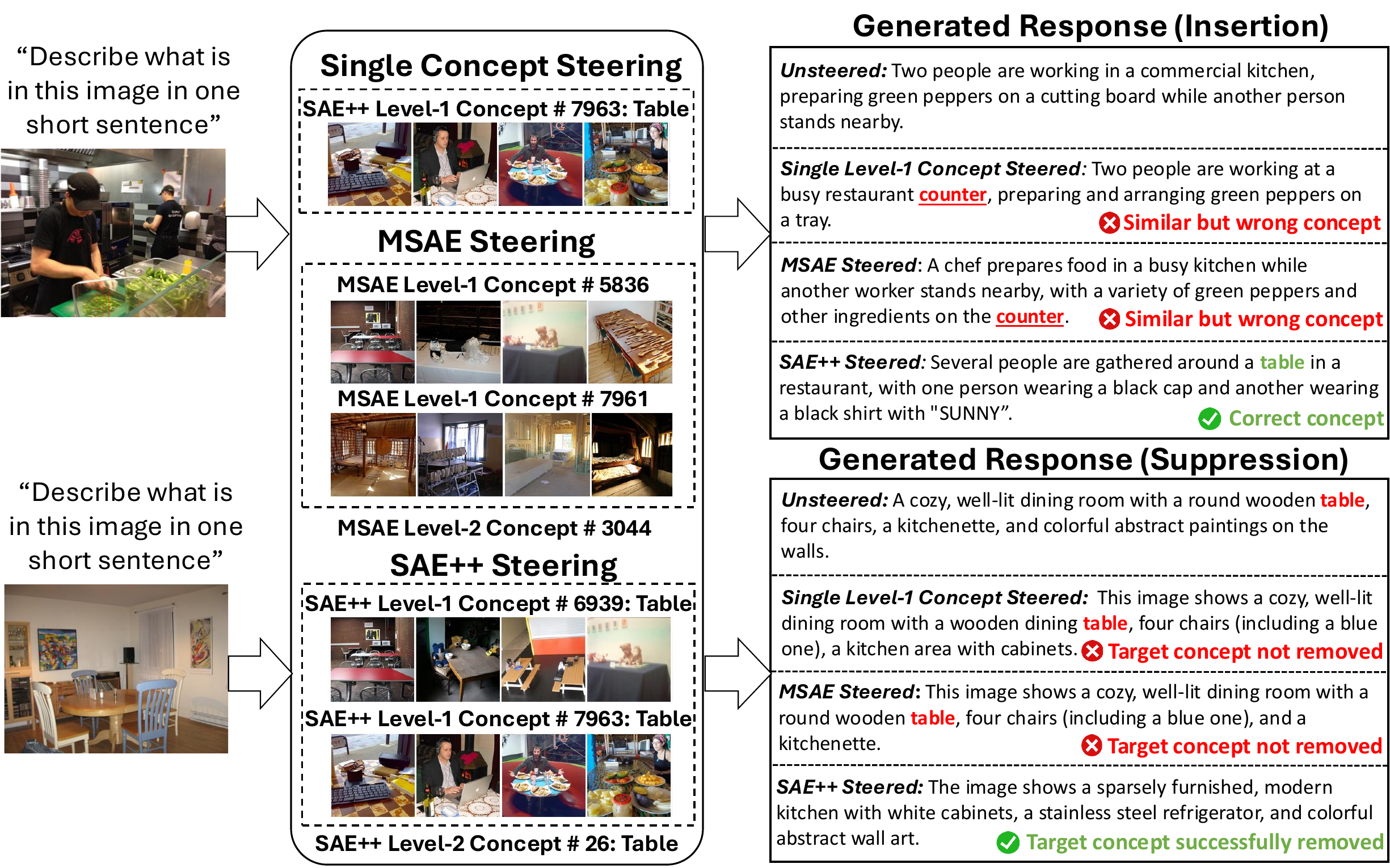}
  \vskip 0.0cm
  \caption{
  Concept steering on \texttt{Qwen3-VL-4B-Instruct}.
  We compare unsteered generation, matched Matryoshka SAE steering, the best individual \modelname{} Level-1 neuron, and full \modelname{} Level-2 clustered steering.
  In both insertion and suppression, \modelname{}'s steering most effectively controls the generated response in terms of the ``\emph{Table}'' concept. 
  }
  \label{fig:steer}
  \vskip -0.3cm
\end{figure*}

\textbf{Ablation Studies.}
\tabref{tab:ablation} evaluates which components are responsible for the gains of \modelnames{}: (1) Restricting gradient flow between the two SAE levels hurts performance: both partial and bilateral stop-gradient variants (\textbf{Partial SG} and \textbf{Bilateral SG}, details in~\appref{appsec:stopgrad}) lead to lower {HMS} scores compared with full joint training (\textbf{{\modelname{} (Full)}}), indicating that the Level-2 objective should shape the Level-1 atoms rather than merely cluster a fixed dictionary. (2) The Stacked SAE (\textbf{Stacked}) performs substantially worse, supporting our theoretical argument that naively re-compressing sparse activation codes is poorly suited for multi-level concept learning. {(3) Replacing the Level-2 SAE with post-hoc clustering on Level-1 decoder weight columns $\w_j$, including Hierarchical Agglomerative Clustering (\textbf{HAC}) and Spectral Clustering (\textbf{SC}), also underperforms our full \modelname{}.
This suggests that the learned Level-2 sparse abstraction is more effective than post-hoc clustering.} 

\begin{table}[!t]
\centering
\caption{
Ablation studies on ImageNet with \texttt{Qwen3-VL-4B-Instruct}.
Results on more HMS metrics are available in~\appref{appsec:full_hms}. 
}
\label{tab:ablation}
\vskip -0.1cm
\setlength{\tabcolsep}{7pt}
\begin{tabular}{lcccccc}
\toprule
\textbf{Metric}
& \textbf{Partial SG}
& \textbf{Bilateral SG}
& \textbf{Stacked}
& \textbf{HAC}
& \textbf{SC}
& \textbf{\modelname{} (Full)} \\
\midrule
\textbf{HMS}$_{\mathrm{med}}$
& 0.8280 & 0.7012 & 0.3367 & \underline{0.8587} & 0.6906 & \textbf{0.9999} \\
\textbf{HMS}$_{\mathrm{mean}}$
& \underline{0.7099} & 0.6333 & 0.3322 & 0.7025 & 0.6490 & \textbf{0.7776} \\
\bottomrule
\end{tabular}
\renewcommand{\arraystretch}{1.0}
\vskip -0.45cm
\end{table}

\section{Conclusion}
\label{sec:contribution}
We introduced \modelname{}, a cascaded SAE framework for hierarchical concept discovery in MLLMs.
By training a Level-2 SAE on Level-1 decoder atoms, \modelname{} learns abstractions over concept directions while avoiding the bottleneck and shared-prefix limitations of stacked and Matryoshka-style designs.
Experiments across three MLLM families and four datasets show that \modelname{} improves hierarchical coherence and enables effective concept-level steering.
Future work includes extending the framework to deeper hierarchies, additional modalities, and language-side representations. See \appref{appsec:limit} for detailed limitations such as reliance on Level-1 SAE quality.

\begin{ack}
We thank all reviewers, the area chair, and the senior area chair for their
valuable comments.

\textbf{Funding.}
HW is supported by the Amazon Faculty Research Award, Microsoft AI \& Society
Fellowship, Nokia Faculty Research Award, NSF CAREER Award IIS-2340125,
NIH grant R01CA297832, and NSF grant IIS-2127918.

\textbf{Competing interests.}
The authors declare no competing interests.
\end{ack}


\bibliographystyle{unsrt}
\bibliography{main}

\input{appendix}
\newpage
\input{checklist.tex}

\end{document}

%% file: defs.tex
\def\b{{\bf b}}

\def\f{{\bf f}}

\def\x{{\bf x}}

\def\u{{\bf u}}

\def\W{{\bf W}}
\def\w{{\bf w}}
\def\0{{\bf 0}}
\def\1{{\bf 1}}

\def\SM{{\mathcal S}}

\def\tha{\mbox{\boldmath$\theta$\unboldmath}}

\def\argmax{\mathop{\rm argmax}}

\newtheorem{remark}{Remark}
\newtheorem{theorem}{Theorem}
\newtheorem{lemma}{Lemma}
\newtheorem{definition}{Definition}

\newtheorem{proposition}{Proposition}

\numberwithin{theorem}{section}
\numberwithin{lemma}{section}
\numberwithin{remark}{section}
\numberwithin{cor}{section}
\numberwithin{proposition}{section}
\numberwithin{definition}{section}

\newcommand{\tabref}[1]{Table~\ref{#1}}
\newcommand{\secref}[1]{Sec.~\ref{#1}}
\newcommand{\appref}[1]{Appendix~\ref{#1}}
\newcommand{\figref}[1]{Fig.~\ref{#1}}

\newcommand{\eqnref}[1]{Eqn. (\ref{#1})}

\renewcommand{\hat}{\widehat}

%% file: appendix.tex
\appendix
\onecolumn

\section{Limitations}
\label{appsec:limit}

\paragraph{Scope of evaluation.}
Our experiments focus on hierarchical visual concept discovery in MLLMs.
We evaluate \modelname on three model families and four visual datasets, but the current study does not fully cover other modalities, such as audio, video, or language-side reasoning representations.
Moreover, our main quantitative evaluation uses HMS to measure whether Level-1 concepts under the same Level-2 parent are semantically coherent.
{HMS does not explicitly penalize duplicated concepts.}
Although HMS captures an important aspect of hierarchical interpretability, it does not exhaust all possible notions of a useful concept hierarchy.
For example, different downstream tasks may require hierarchies organized by function, causality, compositionality, or task relevance rather than visual similarity alone.
Extending \modelname{} to additional modalities, tasks, and hierarchy metrics is an important direction for future work.

\paragraph{Dependence on Level-1 SAE quality and decoder-atom abstraction.}
\modelname assumes that the Level-1 SAE learns meaningful low-level concept directions, and then trains the Level-2 SAE on the corresponding Level-1 decoder atoms.
This design is effective when Level-1 decoder directions are interpretable and semantically organized, but the quality of the learned hierarchy can degrade if the Level-1 SAE contains noisy, dead, highly polysemantic, or poorly reconstructed features.
The method also relies on the interpretation that decoder atoms can serve as inputs for higher-level abstraction.
While our empirical results support this design, we do not exhaustively study how \modelname{} performance changes with different Level-1 SAE architectures, sparsity levels, expansion ratios, target layers, or SAE training datasets.
A more systematic analysis of these design choices would help establish robust defaults for applying \modelname to new models.

\paragraph{Computation, steering, and evaluation limitations.}
\modelname{} introduces additional training-time computation compared with a single flat SAE because it jointly trains a Level-2 SAE over Level-1 decoder atoms.
This overhead is manageable in our experiments, but may become more significant for larger activation dimensions, larger SAE dictionaries, or deeper hierarchies with more than two levels.
For concept steering, we evaluate a limited set of Level-2 concepts and use {four LLM judges} to estimate whether the target concept appears or is removed.
This provides useful evidence that \modelname{} concepts can support group-level interventions, but it is not a complete measure of causal control or safety.
Future work should evaluate steering over more concepts, more prompts, human or task-specific judgments, and possible side effects of interventions on unrelated visual concepts.

\section{Theoretical Analysis}
\label{appsec:theory}

This appendix provides the complete assumptions and proofs for the theoretical claims in~\secref{sec:theory}.
We first prove the sparse-bottleneck failure of naive stacked SAEs, and then prove the shared-prefix error amplification result for Matryoshka SAEs.
{The bottleneck result concerns uniform sparse-code reconstruction, and the shared-prefix result uses an orthonormal linear surrogate. Note that these are structural analyses rather than guarantees for non-convex training.}

\subsection{Impossibility of Naive Stacked SAE Bottlenecks}
\label{appsec:stacked_full_proof}

We analyze a stacked SAE in which a first SAE produces a sparse code in \(\mathbb{R}^n\), and a second SAE compresses this code through a smaller and sparser bottleneck in \(\mathbb{R}^m\).
Let
\[
\mathcal{S}_{n,k_1}(B)
:=
\left\{
\mathbf{z}\in\mathbb{R}^n:
\|\mathbf{z}\|_0\le k_1,\;
\|\mathbf{z}\|_2\le B
\right\}
\]
denote the set of possible Level-1 sparse codes.
A stacked SAE introduces an encoder-decoder pair
\[
g:\mathbb{R}^n\to\mathbb{R}^m,
\qquad
h:\mathbb{R}^m\to\mathbb{R}^n,
\]
with the goal of reconstructing \(\mathbf{z}\in\mathcal{S}_{n,k_1}(B)\) from \(g(\mathbf{z})\).

We assume:
\begin{enumerate}[nosep]
\item \textbf{Uniform reconstruction:}
\[
\sup_{\mathbf{z}\in\mathcal{S}_{n,k_1}(B)}
\|h(g(\mathbf{z}))-\mathbf{z}\|_2 \le \varepsilon .
\]
\item \textbf{Sparse bottleneck:}
\[
g(\mathbf{z})\in\mathcal{S}_{m,k_2}(R)
\quad
\text{for all }\mathbf{z}\in\mathcal{S}_{n,k_1}(B),
\qquad
k_2<k_1 .
\]
\item \textbf{Decoder regularity:}
\(h\) is \(L_h\)-Lipschitz:
\[
\|h(\mathbf{u})-h(\mathbf{v})\|_2
\le
L_h\|\mathbf{u}-\mathbf{v}\|_2 .
\]
\end{enumerate}

\subsubsection{Auxiliary Lemmas}

\begin{definition}[\(\rho\)-separated set]
\label{def:separated_app}
Let \((\mathcal{X},\|\cdot\|_2)\) be a metric space.
A set \(\mathcal{P}\subset\mathcal{X}\) is \(\rho\)-separated if
\[
\|\mathbf{x}-\mathbf{y}\|_2\ge \rho
\quad
\text{for all distinct }\mathbf{x},\mathbf{y}\in\mathcal{P}.
\]
\end{definition}

\begin{lemma}[Uniform reconstruction implies bottleneck separation]
\label{lem:separation_app}
If
\[
\sup_{\mathbf{z}\in\mathcal{S}_{n,k_1}(B)}
\|h(g(\mathbf{z}))-\mathbf{z}\|_2 \le \varepsilon ,
\]
then for all \(\mathbf{z}_1,\mathbf{z}_2\in\mathcal{S}_{n,k_1}(B)\),
\[
\|g(\mathbf{z}_1)-g(\mathbf{z}_2)\|_2
\ge
\frac{\|\mathbf{z}_1-\mathbf{z}_2\|_2-2\varepsilon}{L_h}.
\]
\end{lemma}

\begin{proof}
By the triangle inequality,
\begin{align}
\|\mathbf{z}_1-\mathbf{z}_2\|_2
&\le
\|\mathbf{z}_1-h(g(\mathbf{z}_1))\|_2
+
\|h(g(\mathbf{z}_1))-h(g(\mathbf{z}_2))\|_2 \nonumber\\
&\quad+
\|h(g(\mathbf{z}_2))-\mathbf{z}_2\|_2 .
\end{align}
The first and third terms are at most \(\varepsilon\), and the middle term is at most
\(L_h\|g(\mathbf{z}_1)-g(\mathbf{z}_2)\|_2\).
Rearranging proves the result.
\end{proof}

\begingroup
\color{black}
\begin{lemma}[A separated support packing for sparse codes]
\label{lem:packing_sparse_app}
For $1\le k_1\le n$, $B>0$, and $\rho\in(0,B]$, define
\[
r_\rho=\left\lceil\frac{k_1\rho^2}{2B^2}\right\rceil-1,
\qquad
V_{n,k_1}(r)=\sum_{i=0}^{r}\binom{k_1}{i}\binom{n-k_1}{i},
\]
with binomial coefficients outside their valid range interpreted as zero.
There is a $\rho$-separated set $\mathcal{P}\subset\mathcal{S}_{n,k_1}(B)$ satisfying
\[
|\mathcal{P}|\ge\frac{\binom{n}{k_1}}{V_{n,k_1}(r_\rho)}.
\]
\end{lemma}
\begin{proof}
For each $k_1$-element support $S\subset[n]$, let
$\mathbf{z}_S=(B/\sqrt{k_1})\mathbf{1}_S$. Each vector has norm $B$ and belongs to $\mathcal{S}_{n,k_1}(B)$. Greedily select a support $S$ and remove every remaining support $T$ with $|S\setminus T|\le r_\rho$. Each selection removes at most $V_{n,k_1}(r_\rho)$ supports: for a fixed $i$, there are $\binom{k_1}{i}\binom{n-k_1}{i}$ supports formed by replacing $i$ indices. The selected family therefore has size at least the stated ratio. Any two selected supports satisfy $|S\setminus T|\ge r_\rho+1$, and hence
\[
\|\mathbf{z}_S-\mathbf{z}_T\|_2^2
=\frac{2B^2}{k_1}|S\setminus T|
\ge\frac{2B^2}{k_1}\left\lceil\frac{k_1\rho^2}{2B^2}\right\rceil
\ge\rho^2.
\]
Thus the constructed vectors are separated even when their supports overlap.
\end{proof}
\endgroup

\begin{lemma}[Packing upper bound for sparse bottleneck codes]
\label{lem:packing_bottleneck_app}
Let
\[
\mathcal{S}_{m,k_2}(R)
=
\left\{
\mathbf{u}\in\mathbb{R}^m:
\|\mathbf{u}\|_0\le k_2,\;
\|\mathbf{u}\|_2\le R
\right\}.
\]
Any \(\delta\)-separated subset
\(\mathcal{Q}\subset\mathcal{S}_{m,k_2}(R)\) satisfies
\[
|\mathcal{Q}|
\le
\binom{m}{k_2}
\left(1+\frac{2R}{\delta}\right)^{k_2}.
\]
\end{lemma}

\begin{proof}
Decompose \(\mathcal{S}_{m,k_2}(R)\) by supports.
For each \(T\subset[m]\) with \(|T|=k_2\), define
\[
\mathcal{S}_T(R)
:=
\{\mathbf{u}\in\mathbb{R}^m:\operatorname{supp}(\mathbf{u})\subseteq T,\ \|\mathbf{u}\|_2\le R\}.
\]
Then
\[
\mathcal{S}_{m,k_2}(R)
\subseteq
\bigcup_{T\subset[m],\,|T|=k_2}\mathcal{S}_T(R).
\]
For a \(\delta\)-separated set \(\mathcal{Q}\), let \(\mathcal{Q}_T=\mathcal{Q}\cap\mathcal{S}_T(R)\).
Then
\[
|\mathcal{Q}|
\le
\sum_{T:\,|T|=k_2}|\mathcal{Q}_T|.
\]

For fixed \(T\), \(\mathcal{S}_T(R)\) is isometric to a \(k_2\)-dimensional Euclidean ball of radius \(R\).
If \(Q'\subset B_{k_2}(R)\) is \(\delta\)-separated, then balls of radius \(\delta/2\) centered at points in \(Q'\) are disjoint and contained in \(B_{k_2}(R+\delta/2)\).
A volume comparison gives
\[
|Q'|\operatorname{Vol}(B_{k_2}(\delta/2))
\le
\operatorname{Vol}(B_{k_2}(R+\delta/2)),
\]
and hence
\[
|Q'|
\le
\left(\frac{R+\delta/2}{\delta/2}\right)^{k_2}
=
\left(1+\frac{2R}{\delta}\right)^{k_2}.
\]
Summing over the \(\binom{m}{k_2}\) supports proves the claim.
\end{proof}

\subsubsection{Main Impossibility Result}

\begingroup
\color{black}
\begin{theorem}[Failure of stacked SAEs under a support-packing capacity condition]
\label{thm:stacked_failure_sparse_app}
Assume uniform reconstruction, a $k_2$-sparse bottleneck with $k_2<k_1$, and an $L_h$-Lipschitz decoder, with $L_h>0$ and $\varepsilon\le B/8$.
Let $r_*=\lceil k_1/8\rceil-1$ and use the support-neighborhood count $V_{n,k_1}$ from Lemma~\ref{lem:packing_sparse_app}. If
\begin{align}
\log\binom{m}{k_2}
+k_2\log\!\left(1+\frac{8RL_h}{B}\right)
<\log\binom{n}{k_1}-\log V_{n,k_1}(r_*),
\label{eq:stacked_capacity_condition_app}
\end{align}
then no such encoder-decoder pair can uniformly reconstruct $\mathcal{S}_{n,k_1}(B)$.
\end{theorem}

\begin{proof}
Set $\rho=B/2$. Lemma~\ref{lem:packing_sparse_app} gives a $\rho$-separated set $\mathcal{P}$ with
\begin{align}
\log|\mathcal{P}|\ge\log\binom{n}{k_1}-\log V_{n,k_1}(r_*).
\label{eq:card_P_app}
\end{align}
By Lemma~\ref{lem:separation_app}, $g(\mathcal{P})$ is $\delta$-separated, where
\[
\delta=\frac{\rho-2\varepsilon}{L_h}\ge\frac{B}{4L_h}>0.
\]
The bottleneck packing upper bound gives
\begin{align}
\log|g(\mathcal{P})|\le\log\binom{m}{k_2}
+k_2\log\!\left(1+\frac{8RL_h}{B}\right).
\label{eq:card_gP_app}
\end{align}
Separation makes $g$ injective on $\mathcal{P}$, so $|g(\mathcal{P})|=|\mathcal{P}|$. Combining these two bounds contradicts Eq.~\eqref{eq:stacked_capacity_condition_app}.
\end{proof}

\begin{remark}[Relation to the empirical sparsity configuration]
\label{rem:stacked_common_regime_app}
For fixed $k_1$ and sufficiently large $n$, the support-packing term grows as
$(k_1-r_*)\log n+O(1)$. With $m<n$ and bounded $RL_h/B$, the bottleneck term is at most $k_2\log n+O(1)$. Thus the sufficient condition holds for large $n$ if $k_2<k_1-r_*$. In particular, for the empirical $k_1=6$, $k_2=2$, we have $r_*=0$ and $V_{n,6}(0)=1$: distinct equal-amplitude six-element supports already form a $B/2$-separated packing. The result concerns uniform reconstruction under the stated assumptions.
\end{remark}
\endgroup

\subsection{Full Proofs for Shared-Prefix Error Amplification}
\label{app:shared_prefix_proofs}

We now prove the shared-prefix amplification result used in~\secref{sec:msae_theory}.
The analysis is idealized: it uses an orthonormal linear surrogate to isolate the structural effect of nested-prefix reuse, rather than characterizing nonlinear SAE training dynamics.

\subsubsection{Preliminaries}

For any linear subspace \(U\subseteq\mathbb{R}^d\), let \(P_U\) denote the orthogonal projector onto \(U\), and let \(\|\cdot\|_F\) denote the Frobenius norm.

\begin{lemma}[Projector-distance identity]
\label{lem:projector_distance_identity_app}
Let \(A,B\subseteq\mathbb{R}^d\) be finite-dimensional subspaces. Then
\[
\|P_A-P_B\|_F^2
=
\dim(A)+\dim(B)-2\operatorname{Tr}(P_A P_B).
\]
In particular, if \(\dim(A)=\dim(B)=m\), then
\[
\|P_A-P_B\|_F^2
=
2m-2\operatorname{Tr}(P_A P_B).
\]
\end{lemma}

\begin{proof}
Since orthogonal projectors are symmetric and idempotent,
\begin{align}
\|P_A-P_B\|_F^2
&=
\operatorname{Tr}\bigl((P_A-P_B)^2\bigr) \\
&=
\operatorname{Tr}(P_A^2)+\operatorname{Tr}(P_B^2)-2\operatorname{Tr}(P_A P_B) \\
&=
\operatorname{Tr}(P_A)+\operatorname{Tr}(P_B)-2\operatorname{Tr}(P_A P_B) \\
&=
\dim(A)+\dim(B)-2\operatorname{Tr}(P_A P_B).
\end{align}
The equal-dimensional case follows immediately.
\end{proof}

\begin{lemma}[One-dimensional projector distance]
\label{lem:onedim_projector_bound_app}
Let \(A\) and \(B\) be one-dimensional subspaces of \(\mathbb{R}^d\). Then
\[
\|P_A-P_B\|_F^2\le 2.
\]
\end{lemma}

\begin{proof}
Let \(A=\mathrm{span}\{a\}\) and \(B=\mathrm{span}\{b\}\), where \(a,b\) are unit vectors.
Then \(P_A=aa^\top\), \(P_B=bb^\top\), and
\[
\operatorname{Tr}(P_A P_B)=(a^\top b)^2.
\]
By Lemma~\ref{lem:projector_distance_identity_app},
\[
\|P_A-P_B\|_F^2
=
2-2(a^\top b)^2
\le 2.
\]
\end{proof}

\begin{lemma}[Shared-prefix reuse count]
\label{lem:shared_prefix_count_app}
For nested prefix sizes \(1\le m_1<\cdots<m_L\), the decoder direction at position \(p\) is contained in exactly
\[
C_{\mathrm{MSAE}}(p):=|\{t:m_t\ge p\}|
\]
prefixes. If \(m_{\ell-1}<p\le m_\ell\), with \(m_0:=0\), then
\[
C_{\mathrm{MSAE}}(p)=L-\ell+1.
\]
In particular, \(p\le m_{L-1}\) implies \(C_{\mathrm{MSAE}}(p)\ge2\).
\end{lemma}

\begin{proof}
The \(t\)-th prefix contains exactly the first \(m_t\) decoder directions.
Therefore, the direction at position \(p\) is included in prefix \(t\) if and only if \(p\le m_t\).
Hence the number of prefixes containing that direction is
\[
C_{\mathrm{MSAE}}(p)=|\{t:m_t\ge p\}|.
\]
If \(m_{\ell-1}<p\le m_\ell\), the direction first appears in prefix \(\ell\) and is reused by prefixes \(\ell,\ell+1,\dots,L\), so
\[
C_{\mathrm{MSAE}}(p)=L-\ell+1.
\]
If \(p\le m_{L-1}\), then \(\ell\le L-1\), hence \(C_{\mathrm{MSAE}}(p)\ge2\).
\end{proof}

\subsubsection{From Matryoshka Prefix Reconstruction to Semantic Mismatch}

Recall that Matryoshka SAEs implement hierarchy by training nested index subsets
\[
\mathcal{I}_t=\{1,\ldots,m_t\},
\qquad
m_1<\cdots<m_L,
\]
under the prefix reconstruction objective
\[
\sum_{t=1}^{L}
\left\|
\mathbf{x}
-
\left(
\mathbf{W}_{\mathrm{dec}}^{(:,\mathcal{I}_t)}
\mathbf{f}_{\mathcal{I}_t}
+
\mathbf{b}
\right)
\right\|_2^2 .
\]
Thus, the \(t\)-th prefix is encouraged to reconstruct the input using only the first \(m_t\) decoder columns.

To isolate the effect of prefix reuse, we consider an orthonormal linear surrogate of the Matryoshka decoder.
Let
\[
D=[\mathbf{d}_1,\dots,\mathbf{d}_{m_L}]\in\mathbb{R}^{d\times m_L}
\]
be the ordered decoder directions of the largest prefix, where \(\mathbf{d}_p\) corresponds to the \(p\)-th decoder column.
The \(t\)-th prefix spans the subspace
\[
U_t(D)=\mathrm{span}\{\mathbf{d}_1,\dots,\mathbf{d}_{m_t}\},
\qquad
1\le m_1<\cdots<m_L.
\]
Under the orthonormal linear surrogate, the nonlinear decoder restricted to prefix \(t\) is replaced by the orthogonal projection onto \(U_t(D)\).
Thus, prefix \(t\) reconstructs \(\mathbf{x}\) as \(P_{U_t(D)}\mathbf{x}\), and the Matryoshka prefix reconstruction objective is idealized as
\[
\sum_{t=1}^{L}
\mathbb{E}_{\mathbf{x}}
\left[
\|\mathbf{x}-P_{U_t(D)}\mathbf{x}\|_2^2
\right].
\]
This is the linear surrogate corresponding to the prefix reconstruction loss in the main text; it is used only to isolate shared-prefix reuse, not to model the full nonlinear SAE training dynamics.

We then evaluate semantic alignment by comparing the prefix subspace \(U_t(D)\) with an ideal semantic target.
Let \(\mathcal{M}_t\subseteq\mathbb{R}^d\) denote the target semantic subspace for level \(t\), i.e., the span of the semantic directions that the \(t\)-th prefix is intended to capture.
We set \(\dim(\mathcal{M}_t)=m_t\), matching the dimension of \(U_t(D)\).
The prefix semantic mismatch is
\[
\mathcal{E}_{t}^{\mathrm{MSAE}}(D)
:=
\|P_{U_t(D)}-P_{\mathcal{M}_t}\|_F^2,
\]
and the total prefix semantic mismatch is
\[
\mathcal{E}^{\mathrm{tot}}_{\mathrm{MSAE}}(D)
:=
\sum_{t=1}^{L}
\mathcal{E}_{t}^{\mathrm{MSAE}}(D).
\]
This projector mismatch is not the original reconstruction objective.
It is a semantic diagnostic defined on top of the linear surrogate: it measures whether the subspace used by prefix \(t\) aligns with the semantic subspace that level \(t\) is intended to represent.

\subsubsection{Normalized Local Semantic Error}

We now define the local semantic errors used in the theorem.
Suppose there are \(q\) nuisance decoder directions in the Matryoshka ordering.
Formally, write
\[
\mathbf{d}_{p_i}=\mathbf{n}_i,
\qquad
i=1,\dots,q,
\]
where the \(\mathbf{n}_i\)'s are mutually orthonormal and each \(\mathbf{n}_i\) is orthogonal to every target semantic subspace \(\mathcal{M}_t\).
For every affected prefix \(t\) with \(m_t\ge p_i\), define
\[
U_t^{(-i)}(D)
:=
U_t(D)\cap\mathrm{span}\{\mathbf{n}_i\}^{\perp}.
\]
This removes the one-dimensional nuisance direction \(\mathbf{n}_i\) from the prefix subspace.

\paragraph{Semantic repair direction assumption.}
For every nuisance direction \(\mathbf{n}_i\) and affected prefix \(t\) with \(m_t\ge p_i\), assume there exists a unit vector
\[
\mathbf{s}_{i,t}
\in
\mathcal{M}_t\cap \bigl(U_t^{(-i)}(D)\bigr)^\perp.
\]
Define the repaired prefix subspace
\[
\widetilde{U}_t^{(-i)}(D)
:=
U_t^{(-i)}(D)\oplus \mathrm{span}\{\mathbf{s}_{i,t}\}.
\]
This repair is an analysis device: after removing a non-semantic direction, it fills the freed one-dimensional slot with a valid semantic direction while preserving the prefix dimension.

The local Matryoshka semantic error caused by \(\mathbf{n}_i\) in prefix \(t\) is
\[
e_{i,t}^{\mathrm{MSAE}}
:=
\mathcal{E}_{t}^{\mathrm{MSAE}}(D)
-
\|P_{\widetilde{U}_t^{(-i)}(D)}-P_{\mathcal{M}_t}\|_F^2.
\]
Since the same decoder direction may appear in several prefixes, we normalize by its reuse count:
\[
\overline{\mathcal{E}}_{\mathrm{MSAE}}
:=
\sum_{i=1}^{q}
\frac{1}{C_{\mathrm{MSAE}}(p_i)}
\sum_{t:m_t\ge p_i}e_{i,t}^{\mathrm{MSAE}}.
\]

For \modelname{}, each Level-1 atom \(\mathbf{w}_j\) has an ideal semantic parent, denoted by \(c(j)\).
Let \(\mathcal{P}_{c(j)}\) be the one-dimensional subspace corresponding to this ideal parent, and let \(\widehat{\mathcal{P}}_j\) be the parent subspace assigned by the Level-2 code of \(\mathbf{w}_j\).
For a locally erroneous atom assignment \(j\), define
\[
e_j^{\mathrm{\modelname}}
:=
\|P_{\widehat{\mathcal{P}}_j}-P_{\mathcal{P}_{c(j)}}\|_F^2.
\]
For a set \(\mathcal{B}\) of locally erroneous atom assignments, define
\[
\overline{\mathcal{E}}_{\mathrm{\modelname}}(\mathcal{B})
:=
\sum_{j\in\mathcal{B}}e_j^{\mathrm{\modelname}}.
\]
No prefix-reuse normalization is needed for \modelname{}, because each atom-level error appears in only one atom-wise Level-2 term.

\begin{theorem}[Prefix reuse amplifies Matryoshka semantic error]
\label{thm:shared_prefix_app}
Under the setup above, the total prefix semantic mismatch satisfies
\[
\mathcal{E}^{\mathrm{tot}}_{\mathrm{MSAE}}(D)
\ge
2\sum_{i=1}^{q}C_{\mathrm{MSAE}}(p_i).
\]
Moreover, under the normalized local comparison,
\[
\overline{\mathcal{E}}_{\mathrm{MSAE}}\ge 2q.
\]
For \modelname{}, if \(|\mathcal{B}|\le q\), then
\[
\overline{\mathcal{E}}_{\mathrm{\modelname}}(\mathcal{B})
\le
2q
\le
\overline{\mathcal{E}}_{\mathrm{MSAE}}.
\]
\end{theorem}

\begin{proof}
We first prove the total prefix semantic mismatch bound.
Fix a prefix \(t\), and let
\[
I_t:=\{i:m_t\ge p_i\},
\qquad
q_t:=|I_t|.
\]
The set \(I_t\) indexes the nuisance directions contained in \(U_t(D)\).
Because these nuisance directions are mutually orthonormal and orthogonal to \(\mathcal{M}_t\), they occupy \(q_t\) orthogonal dimensions of \(U_t(D)\) with zero overlap with \(\mathcal{M}_t\).
Therefore the trace overlap between \(U_t(D)\) and \(\mathcal{M}_t\) is at most \(m_t-q_t\):
\[
\operatorname{Tr}(P_{U_t(D)}P_{\mathcal{M}_t})
\le
m_t-q_t.
\]
Using Lemma~\ref{lem:projector_distance_identity_app},
\[
\|P_{U_t(D)}-P_{\mathcal{M}_t}\|_F^2
=
2m_t-2\operatorname{Tr}(P_{U_t(D)}P_{\mathcal{M}_t})
\ge
2q_t.
\]
Summing over all prefixes gives
\[
\mathcal{E}^{\mathrm{tot}}_{\mathrm{MSAE}}(D)
=
\sum_{t=1}^{L}
\|P_{U_t(D)}-P_{\mathcal{M}_t}\|_F^2
\ge
\sum_{t=1}^{L}2q_t.
\]
Since
\[
\sum_{t=1}^{L}q_t
=
\sum_{t=1}^{L}|\{i:m_t\ge p_i\}|
=
\sum_{i=1}^{q}|\{t:m_t\ge p_i\}|
=
\sum_{i=1}^{q}C_{\mathrm{MSAE}}(p_i),
\]
we obtain
\[
\mathcal{E}^{\mathrm{tot}}_{\mathrm{MSAE}}(D)
\ge
2\sum_{i=1}^{q}C_{\mathrm{MSAE}}(p_i).
\]

Next, we prove the normalized Matryoshka local-error bound.
Fix a nuisance direction \(\mathbf{n}_i\) and an affected prefix \(t\) with \(m_t\ge p_i\).
Let
\[
U:=U_t(D),
\qquad
\widetilde{U}:=\widetilde{U}_t^{(-i)}(D),
\qquad
M:=\mathcal{M}_t.
\]
Both \(U\) and \(\widetilde{U}\) have dimension \(m_t\), and \(M\) also has dimension \(m_t\).
By Lemma~\ref{lem:projector_distance_identity_app},
\[
e_{i,t}^{\mathrm{MSAE}}
=
2\operatorname{Tr}(P_{\widetilde{U}}P_M)
-
2\operatorname{Tr}(P_U P_M).
\]
The subspaces \(U\) and \(\widetilde{U}\) differ only by replacing \(\mathbf{n}_i\) with \(\mathbf{s}_{i,t}\).
The shared component \(U_t^{(-i)}(D)\) contributes equally to both trace terms and cancels.
Because \(\mathbf{n}_i\perp M\), the removed direction contributes zero trace overlap with \(M\).
Because \(\mathbf{s}_{i,t}\in M\) is a unit vector, the repair direction contributes one unit of trace overlap with \(M\).
Therefore
\[
\operatorname{Tr}(P_{\widetilde{U}}P_M)
-
\operatorname{Tr}(P_U P_M)
=
1,
\]
and hence
\[
e_{i,t}^{\mathrm{MSAE}}=2.
\]
Averaging over the \(C_{\mathrm{MSAE}}(p_i)\) affected prefixes gives
\[
\frac{1}{C_{\mathrm{MSAE}}(p_i)}
\sum_{t:m_t\ge p_i}e_{i,t}^{\mathrm{MSAE}}
=
2.
\]
Summing over \(i=1,\dots,q\) yields
\[
\overline{\mathcal{E}}_{\mathrm{MSAE}}
=
2q,
\]
and in particular \(\overline{\mathcal{E}}_{\mathrm{MSAE}}\ge2q\).

Finally, for \modelname{}, each \(e_j^{\mathrm{\modelname}}\) is the Frobenius distance between two one-dimensional projectors.
By Lemma~\ref{lem:onedim_projector_bound_app},
\[
e_j^{\mathrm{\modelname}}\le2.
\]
Thus, if \(|\mathcal{B}|\le q\),
\[
\overline{\mathcal{E}}_{\mathrm{\modelname}}(\mathcal{B})
=
\sum_{j\in\mathcal{B}}e_j^{\mathrm{\modelname}}
\le
2|\mathcal{B}|
\le
2q.
\]
Combining the inequalities gives
\[
\overline{\mathcal{E}}_{\mathrm{\modelname}}(\mathcal{B})
\le
2q
\le
\overline{\mathcal{E}}_{\mathrm{MSAE}}.
\]
\end{proof}

\begin{remark}[Scope of the idealization]
The result above is not a universal claim about nonlinear Matryoshka SAE optimization.
It isolates one structural effect of nested prefixes under an orthonormal linear surrogate:
a semantic error in an early decoder direction is reused by every downstream prefix containing that direction.
The normalized comparison removes this repeated counting, while the total prefix semantic mismatch keeps it.
\end{remark}

\subsection{Cluster-consistent Coding Compatibility}
\label{appsec:cluster_consistent_coding}

The main text focuses on avoiding stacked bottlenecks and shared-prefix amplification.
For completeness, we also record a simple representational compatibility property of the \modelname{} second stage.

\begin{proposition}[Cluster-consistent coding compatibility]
\label{prop:csae_subset_uniform_app}
{Assume $K\le n_2$ and that each input atom has a unique semantic parent. Under the clustered semantic atom model}
\[
\mathbf{w}_j=\boldsymbol{\mu}_{c(j)}+\boldsymbol{\xi}_j,
\qquad
\|\boldsymbol{\xi}_j\|_2\le\sigma.
\]
Then there exists a Level-2 sparse encoder \(g_2\) such that atoms with the same semantic parent have identical Level-2 support.
Consequently, for every nonempty subset \(S\subseteq\{1,\dots,n_1\}\) with at least one same-parent pair,
\[
\mathcal{E}_{\mathrm{\modelname}}^{\mathrm{supp}}(S)=0.
\]
\end{proposition}

\begin{proof}
For each semantic prototype index \(c\in\{1,\dots,K\}\), choose one Level-2 latent unit and denote its one-hot code by \(\mathbf{e}_c\).
Define
\[
g_2(\mathbf{w}_j):=\mathbf{e}_{c(j)}.
\]
If \(c(i)=c(j)\), then
\[
g_2(\mathbf{w}_i)=\mathbf{e}_{c(i)}=\mathbf{e}_{c(j)}=g_2(\mathbf{w}_j),
\]
so
\[
\mathrm{supp}(g_2(\mathbf{w}_i))
=
\mathrm{supp}(g_2(\mathbf{w}_j)).
\]
Thus every same-parent pair contributes zero to the support-inconsistency indicator, and hence
\[
\mathcal{E}_{\mathrm{\modelname}}^{\mathrm{supp}}(S)=0.
\]
This is an achievability statement about representational compatibility, not an optimization guarantee.
\end{proof}

\section{Generalization to an $L$-Level \modelname}\label{appsec:generalized_saepp}

\subsection{General $L$-Level \modelname}

While the description in~\secref{sec:saepp} focuses on a two-level hierarchy for clarity, the proposed \modelname framework naturally extends to an arbitrary number of hierarchical levels. In this section, we describe the general formulation of an $L$-level cascaded sparse autoencoder, which enables progressively higher-order semantic abstraction.

\textbf{Architecture Overview.}
Let Level-0 denote the original LLM activation space, with input activations $\mathbf{x} \in \mathbb{R}^d$. The model consists of $L$ SAEs arranged in a cascade. {Level 1 operates on the original activations.} Each Level-$\ell$ SAE ({$\ell = 2, \dots, L$}) operates on the decoder weights learned at Level-$(\ell-1)$, thereby inducing a hierarchy of concepts from low-level concepts to high-level semantic structures.

We denote by $n_\ell$ the number of latent units at Level $\ell$, and by
\begin{equation}
    \mathbf{W}^{(\ell)}_{\text{dec}} = [\mathbf{w}^{(\ell)}_1, \dots, \mathbf{w}^{(\ell)}_{n_\ell}]
\end{equation}
the decoder matrix of the Level-$\ell$ SAE, where each column $\mathbf{w}^{(\ell)}_j \in \mathbb{R}^d$ represents a concept direction in the original activation space.

\textbf{Level-$\ell$ SAE.}
For {$\ell \ge 2$}, the Level-$\ell$ SAE consists of an encoder-decoder pair $(g_\ell, h_\ell)$. Its input is given by the decoder weight columns of the previous level, $\{\mathbf{w}^{(\ell-1)}_j\}_{j=1}^{n_{\ell-1}}$. For each such weight column, the forward pass is defined as
\begin{align}
    \mathbf{f}^{(\ell)} &= g_\ell(\mathbf{w}^{(\ell-1)}_j)
    = \sigma\!\left(\mathbf{W}^{(\ell)}_{\text{enc}}(\mathbf{w}^{(\ell-1)}_j - \mathbf{b}^{(\ell)})\right), \\
    \widehat{\mathbf{w}}^{(\ell-1)}_j
    &= h_\ell(\mathbf{f}^{(\ell)})
    = \mathbf{W}^{(\ell)}_{\text{dec}} \mathbf{f}^{(\ell)} + \mathbf{b}^{(\ell)} .
\end{align}
Here, $\mathbf{W}^{(\ell)}_{\text{dec}}$ encodes the concept directions at Level $\ell$, while the sparse code $\mathbf{f}^{(\ell)} \in \mathbb{R}^{n_\ell}$ represents the assignment of a lower-level concept to higher-level abstractions.

\textbf{Level-$\ell$ Objective.}
The training objective for the Level-$\ell$ SAE {($\ell\ge2$)} is defined as
\begin{align}
    \mathcal{L}_\ell\!\left(\mathbf{W}^{(\ell-1)}_{\text{dec}}, \tha_\ell\right)
    =
    \frac{1}{n_{\ell-1}} \sum_{j=1}^{n_{\ell-1}}
    \Big(
    \big\|\mathbf{w}^{(\ell-1)}_j
    - h_\ell(g_\ell(\mathbf{w}^{(\ell-1)}_j))\big\|_2^2
    + \lambda_\ell \mathcal{S}(g_\ell(\mathbf{w}^{(\ell-1)}_j))
    \Big),
\end{align}
where $\mathcal{S}(\cdot)$ is typically an $\ell_1$ sparsity penalty, and $\tha_\ell = \{\mathbf{W}^{(\ell)}_{\text{enc}}, \mathbf{W}^{(\ell)}_{\text{dec}}, \mathbf{b}^{(\ell)}\}$. 

\textbf{Joint End-to-End Optimization.}
The full $L$-level \modelname is trained end-to-end with a weighted sum of reconstruction and sparsity objectives across all levels:
\begin{align}
    \mathcal{L}_{\text{final}}
    =
    \mathbb{E}_{\mathbf{x}}\!\left[\mathcal{L}_1(\mathbf{x}, \theta_1)\right]
    + \sum_{\ell=2}^{L} \alpha_\ell \,
    \mathcal{L}_\ell\!\left(\mathbf{W}^{(\ell-1)}_{\text{dec}}, \theta_\ell\right),
\end{align}
where $\alpha_\ell$ controls the relative importance of semantic abstraction at Level $\ell$.

{\textbf{Dynamic Masking of Dead Latents (General $L$-Level Case).}
To mitigate computational inefficiency arising from inactive (``dead'') latents in deep hierarchies, we generalize the dynamic masking strategy to all higher levels of \modelname. At each training step $t$ with mini-batch $\mathcal{B}$, the Level-$\ell$ SAE ($\ell \ge 2$) is optimized exclusively over the subset of Level-$(\ell-1)$ decoder atoms that are active under the current forward pass of the immediately preceding level, while dead latents are ignored.}

For Level $\ell = 2$, the active set is defined using the original LLM activations $\mathbf{x}$:
\begin{equation}
    \mathcal{M}_t^{(2)}
    =
    \left\{
    \mathbf{w}^{(1)}_j
    \ \middle| \
    \sum_{\mathbf{x} \in \mathcal{B}}
    [g_1(\mathbf{x})]_j > 0
    \right\}.
\end{equation}

For higher levels $\ell > 2$, the activity of Level-$(\ell-1)$ latents is instead determined by the encoder of the Level-$(\ell-1)$ SAE applied to Level-$(\ell-2)$ decoder weight columns $\mathbf{W}^{(\ell-2)}_{\text{dec}} = [\mathbf{w}^{(\ell-2)}_1, \dots, \mathbf{w}^{(\ell-2)}_{n_{\ell-2}}]$:
\begin{equation}
    \mathcal{M}_t^{(\ell)}
    =
    \left\{
    \mathbf{w}^{(\ell-1)}_j
    \ \middle| \
    \sum_{\mathbf{w}^{(\ell-2)} \in \mathbf{W}^{(\ell-2)}_{\text{dec}}}
    \big[g_{\ell-1}(\mathbf{w}^{(\ell-2)})\big]_j > 0
    \right\}.
\end{equation}

Here, $[g_{\ell-1}(\cdot)]_j$ denotes the $j$-th entry of the Level-$(\ell-1)$ encoder output.
{By restricting each Level-$\ell$ objective to $\mathcal{M}_t^{(\ell)}$, we ensure that higher-level concepts are learned strictly from valid, active semantic directions at the preceding level, improving both computational efficiency and hierarchical semantic grounding in deep cascades.}

\textbf{Hierarchical Interpretation.}
Under this formulation, each Level-$\ell$ SAE induces a parent assignment mapping from Level-$(\ell-1)$ concepts to Level-$\ell$ concepts via the encoder activations $g_\ell(\cdot)$. Stacking multiple levels therefore yields a deep hierarchy of increasingly abstract, sparse, and semantically organized concepts, with Level-1 capturing atomic features and higher levels encoding progressively coarser semantic groupings.

This general formulation subsumes the two-level \modelname as a special case with $L=2$, while providing a principled and scalable framework for learning deep hierarchical concept structures from multimodal LLM activations.


\subsection{Hierarchical Mono-Semanticity (HMS) Score: General $L$-Level Formulation}

To evaluate the semantic coherence of hierarchical concepts discovered by \modelname at multiple abstraction depths, we generalize the Hierarchical Mono-Semanticity (HMS) score from the two-level setting to an $L$-level hierarchy. This extension enables quantitative assessment of semantic alignment between concepts at adjacent levels throughout the hierarchy.

\textbf{Level-1 Concept Representations.}
As in the two-level case, we begin by constructing semantic representations for Level-1 (lowest-level) SAE neurons. For each Level-1 hidden neuron $j$ (i.e., each entry of $\mathbf{f}^{(1)}$), we define a semantic representation $\mathbf{R}^{(1)}_j \in \mathbb{R}^{d_e}$ using a set of $N$ input images.

Let $\{\mathbf{E}_i\}_{i=1}^{N}$ denote image embeddings extracted from a fixed vision backbone, where $\mathbf{E}_i \in \mathbb{R}^{d_e}$ corresponds to image $i$. Let $v_{i,j}$ be the activation of Level-1 neuron $j$ for the MLLM activation associated with image $i$. We normalize activations using min--max normalization:
\begin{align}
    a_{i,j} = \frac{v_{i,j} - \min(\mathbf{v}_j)}{\max(\mathbf{v}_j) - \min(\mathbf{v}_j)+\epsilon},
\end{align}
where $\mathbf{v}_j = \{v_{i,j}\}_{i=1}^{N}$. The Level-1 concept representation is then defined as
\begin{align}
    \mathbf{R}^{(1)}_j
    =
    \frac{\sum_{i=1}^{N} a_{i,j} \mathbf{E}_i}{\sum_{i=1}^{N} a_{i,j}+\epsilon}.
\end{align}

\textbf{General $L$-Level HMS Definition.}
Assume an $L$-level hierarchical structure with Level-$\ell$ neurons indexed by
\(
j_\ell \in \{1,\dots,n_\ell\}
\)
for $\ell = 1,\dots,L$. For each level $\ell \ge 2$, let
\(
\pi_\ell(j_{\ell-1}) \in \{1,\dots,n_\ell\}
\)
denote the parent assignment function that maps a Level-$(\ell-1)$ neuron to a Level-$\ell$ neuron. Each neuron at every level corresponds to a learned semantic concept.

For a given Level-$\ell$ neuron $k$, define its set of children at Level-$(\ell-1)$ as
\begin{align}
    \mathcal{C}^{(\ell)}_k
    =
    \{ j \mid \pi_\ell(j) = k \},
\end{align}
and let $m = |\mathcal{C}^{(\ell)}_k|$.
{Here $\epsilon>0$ is a numerical stability constant.}

The HMS score for a Level-$\ell$ neuron $k$ is defined as the mean pairwise cosine similarity among the semantic representations of its Level-$(\ell-1)$ children:
\begin{align}
    \text{HMS}^{(\ell)}(k)
    &=
    \frac{2}{m(m-1)}
    \sum_{\substack{u,v \in \mathcal{C}^{(\ell)}_k \\ u < v}}
    \frac{
        \big(\mathbf{R}^{(\ell-1)}_u\big)^\top
        \mathbf{R}^{(\ell-1)}_v
    }{
        \|\mathbf{R}^{(\ell-1)}_u\|_2
        \,
        \|\mathbf{R}^{(\ell-1)}_v\|_2
    }.
\end{align}

An HMS score close to $1.0$ indicates that the Level-$\ell$ neuron groups semantically coherent and monosemantic Level-$(\ell-1)$ concepts.

\textbf{Semantic Representations at Higher Levels.}
For $\ell \ge 2$, the semantic representation of a Level-$\ell$ neuron $k$ is defined recursively as the mean of its children’s representations:
\begin{align}
    \mathbf{R}^{(\ell)}_k
    =
    \frac{1}{|\mathcal{C}^{(\ell)}_k|}
    \sum_{j \in \mathcal{C}^{(\ell)}_k}
    \mathbf{R}^{(\ell-1)}_j,
\end{align}
enabling HMS evaluation to propagate consistently across all levels of the hierarchy.

\textbf{Instantiation for \modelname.}
In \modelname, the parent assignment functions $\{\pi_\ell\}_{\ell=2}^{L}$ are induced intrinsically by the cascaded SAEs. Specifically, for $\ell \ge 2$, each Level-$(\ell-1)$ neuron $j$ is associated with its decoder weight column $\mathbf{w}^{(\ell-1)}_j$, which is passed through the Level-$\ell$ encoder $g_\ell$. The parent index is defined as
\begin{align}
    \pi_\ell(j)
    =
    \argmax_{r \in \{1,\dots,n_\ell\}}
    \big[g_\ell(\mathbf{w}^{(\ell-1)}_j)\big]_r .
\end{align}
Under this definition, the Level-$\ell$ HMS scores, $\text{HMS}^{(\ell)}$, directly quantify the semantic coherence of the hierarchical clustering learned end-to-end at each level of abstraction.

\textbf{Instantiation for Baseline Architectures.}
For baseline models that do not explicitly define multi-level hierarchies, parent assignments $\pi_\ell(\cdot)$ are constructed post hoc using co-activation statistics between adjacent levels. The HMS score is then computed using the same general formulation above, enabling a unified and fair comparison across different SAE variants.

\subsection{Key Differences between \modelname and Hyper-Networks}\label{appsec:hypernet}
Although \modelname introduces a hierarchical structure over learned features, it is fundamentally different from hyper-network-based approaches~\cite{ha2016hypernetworks}. Hyper-networks explicitly \emph{generate or modulate} the parameters of a target network through a separate conditioning network, tightly coupling representation learning with parameter synthesis within a single forward pass. 

In contrast, \modelname treats the learned decoder atoms (i.e., weight columns) of the lower-level SAE as first-class \emph{data} and applies sparse dictionary learning recursively to model semantic structure \emph{among features themselves}. Importantly, the higher-level SAE does not act as a parameter generator or controller. Even though gradients are allowed to flow through $\mathcal{L}_2$ to the lower-level decoder in our final formulation, this interaction arises solely from a reconstruction-based objective over {learned semantic directions}, rather than from conditional parameterization. 

As a result, \modelname learns stable, reusable abstractions over low-level concept directions via sparsity and reconstruction constraints, yielding an explicit and interpretable concept hierarchy with well-defined parent-child relationships -- distinct from the implicit, task-driven parameter conditioning characteristic of hyper-networks -- and enabling principled evaluation through hierarchical mono-semanticity metrics.

\begingroup
\color{black}
\section{Additional Evaluation and Training Diagnostics}
\label{appsec:additional_evaluation}

\subsection{A Shared Parent-Assignment Control}
\label{appsec:neutral_hac}
To compare all methods under a common parent-assignment protocol, we replace their native assignments with average-linkage Hierarchical Agglomerative Clustering (HAC), using cosine distance on unit-normalized alive Level-1 decoder atoms. The number of clusters is fixed and shared across methods. We evaluate the resulting groups with the same HMS definition and semantic representations, excluding groups with fewer than two children.

\begin{table}[htbp]
\centering
\caption{HMS under a shared HAC parent-assignment protocol. All native parent assignments are discarded. Higher is better; the best result in each row is bold.}
\label{tab:neutral_hac}
\small
\setlength{\tabcolsep}{2.1pt}
\begin{tabular}{llrrrrrrrrr}
\toprule
MLLM & Data & BTK & TopK & ReLU & P-Ann. & GSAE & JReLU & MSAE & Stack. & \modelname{} \\
\midrule
Qwen3 & ImageNet & 0.165 & 0.163 & 0.442 & 0.177 & 0.218 & 0.572 & 0.175 & 0.310 & \textbf{0.703} \\
      & COCO     & 0.119 & 0.124 & 0.049 & 0.223 & 0.117 & 0.347 & 0.706 & 0.157 & \textbf{0.933} \\
Gemma-3 & ImageNet & 0.146 & 0.348 & 0.079 & 0.410 & 0.469 & 0.608 & 0.201 & 0.173 & \textbf{0.722} \\
        & COCO     & 0.215 & 0.266 & 0.058 & 0.181 & 0.260 & 0.306 & 0.787 & 0.192 & \textbf{0.852} \\
\bottomrule
\end{tabular}
\end{table}

As shown in Table~\ref{tab:neutral_hac}, \modelname{} achieves the highest HMS in all four settings under the shared HAC protocol. This control applies HAC to every trained dictionary, including \modelname{}, whereas the clustering ablation replaces its Level-2 SAE.

\subsection{Cluster Sizes and Coverage}
\label{appsec:cluster_coverage}
HMS is defined only for parents with at least two children. For the native evaluation, we report the number of these HMS-eligible parents and their child-count distribution. Atom coverage is the fraction of the complete Level-1 dictionary assigned to eligible parents, with dead atoms and atoms in singleton groups excluded from the numerator. The following statistics use ImageNet with Qwen3-VL-4B-Instruct and the native parent assignments, rather than the shared HAC control above.

\begin{table}[htbp]
\centering
\caption{Native HMS-eligible cluster statistics on ImageNet with Qwen3-VL-4B-Instruct. Cluster-size statistics are computed over eligible clusters. Coverage is a proportion of the complete Level-1 dictionary.}
\label{tab:cluster_coverage}
\begin{tabular}{lrrr}
\toprule
Method & Eligible clusters & Size min/median/max & Atom coverage \\
\midrule
BTK         & 143 & 2/68/4112 & 0.656 \\
TopK        & 77  & 2/2/810   & 0.598 \\
ReLU        & 18  & 3/8/922   & 0.219 \\
P-Ann.      & 15  & 2/74/7774 & 0.547 \\
GSAE        & 25  & 2/4/1048  & 0.522 \\
JReLU       & 74  & 2/2/2083  & 0.230 \\
MSAE        & 176 & 2/3/148   & 0.623 \\
Stacked     & 249 & 2/8/14    & 0.101 \\
\modelname{} & 162 & 2/4/763 & 0.565 \\
\bottomrule
\end{tabular}
\end{table}

For \modelname{}, the 162 eligible clusters cover approximately 11,300 atoms, or 56.5\% of the 20,000-atom dictionary. Their median and maximum sizes are four and 763. Since aggregate HMS weights each eligible parent equally, these statistics provide context for the reported coherence scores.

\subsection{Steering Evaluation with Multiple Judges}
\label{appsec:multi_judge}
We re-evaluate the same steering outputs with three additional LLM judges: Gemini-3-Flash, Nemotron-30B, and Opus-4.8, using the target-concept presence protocol of the original Gemini-2.5-Flash evaluation. NS denotes no steering, Rand. denotes the random steering baseline, and Single denotes steering the strongest individual Level-1 concept. Judge Avg. is the unweighted mean across the four judges, rounded to one decimal place.

Table~\ref{tab:steering} in the main text reports the insertion and suppression results for all four judges and their average.

\subsubsection{Concept-level Variability and Paired Comparisons}
\label{appsec:steering_variability}
Using Opus-4.8 judgments, we report the mean and standard deviation of success rates across the ten target concepts.

\begin{table}[htbp]
\centering
\caption{Steering success rates (\%): mean $\pm$ standard deviation across ten target concepts, judged by Opus-4.8.}
\label{tab:steering_concept_sd}
\begin{tabular}{lrr}
\toprule
Method & Insertion & Suppression \\
\midrule
NS           & $4.7\pm8.9$   & $0.0\pm0.0$ \\
Rand.        & $5.2\pm10.0$  & $9.6\pm6.4$ \\
MSAE         & $41.0\pm9.1$  & $22.6\pm7.1$ \\
Single       & $60.9\pm12.2$ & $30.3\pm7.2$ \\
\modelname{} & $72.6\pm12.9$ & $35.7\pm7.8$ \\
\bottomrule
\end{tabular}
\end{table}

We perform paired $t$-tests on the ten concept-level scores, pairing methods by target concept. Table~\ref{tab:steering_pvalues} reports the resulting $p$-values; each comparison favors \modelname{} with $p<0.05$.

\begin{table}[htbp]
\centering
\caption{Paired $t$-test $p$-values over ten matched concept-level scores (Opus-4.8).}
\label{tab:steering_pvalues}
\begin{tabular}{lrr}
\toprule
\modelname{} versus & Insertion & Suppression \\
\midrule
NS     & $<0.001$ & $<0.001$ \\
Rand.  & $<0.001$ & $<0.001$ \\
Single & $0.007$  & $0.011$ \\
MSAE   & $<0.001$ & $0.001$ \\
\bottomrule
\end{tabular}
\end{table}

\subsubsection{Observed Off-target Interference}
\label{appsec:off_target}
We sample 20 unrelated concepts and use Claude-Sonnet-5 with ten representative keywords per concept to judge their presence in captions. The off-target flip rate measures changes in concept-presence judgments between matched steered and unsteered captions. We report the mean and standard deviation across the 20 concepts as proportions.

\begin{table}[htbp]
\centering
\caption{Off-target flip rates: mean $\pm$ standard deviation across 20 sampled unrelated concepts. Lower is better; NS is zero by construction.}
\label{tab:off_target}
\begin{tabular}{lrr}
\toprule
Method & Insertion & Suppression \\
\midrule
NS           & $0.000\pm0.000$ & $0.000\pm0.000$ \\
Rand.        & $0.002\pm0.002$ & $0.001\pm0.002$ \\
MSAE         & $0.005\pm0.003$ & $0.003\pm0.002$ \\
Single       & $0.005\pm0.002$ & $0.003\pm0.003$ \\
\modelname{} & $0.004\pm0.003$ & $0.003\pm0.004$ \\
\bottomrule
\end{tabular}
\end{table}

\modelname{} has off-target flip rates comparable to the learned baselines. The sampled unrelated concepts have low baseline prevalence, so this check covers a limited set of potential side effects.

\subsection{Reconstruction at Matched Level-1 Sparsity}
\label{appsec:matched_fvu}
We compare \modelname{} and MSAE on ImageNet with Qwen3-VL-4B-Instruct at the same Level-1 dictionary dimensionality and sparsity $k_1$. Fraction of Variance Unexplained (FVU) measures normalized reconstruction error; lower is better.

\begin{table}[htbp]
\centering
\caption{Level-1 FVU at matched dictionary dimensionality and sparsity on ImageNet with Qwen3-VL-4B-Instruct.}
\label{tab:matched_fvu}
\begin{tabular}{rrr}
\toprule
Level-1 sparsity $k_1$ & \modelname{} FVU & MSAE FVU \\
\midrule
3  & 0.006266 & 0.005837 \\
6  & 0.004639 & 0.004800 \\
12 & 0.003432 & 0.003320 \\
\bottomrule
\end{tabular}
\end{table}

The methods have comparable reconstruction quality at matched sparsity: MSAE has lower FVU at $k_1=3$ and 12, while \modelname{} has lower FVU at $k_1=6$.

\subsection{Empirical Training Stability}
\label{appsec:empirical_convergence}
Table~\ref{tab:empirical_convergence} reports training diagnostics for an ImageNet + Qwen3-VL-4B-Instruct run with TopK sparsities $k_1=6$ and $k_2=2$.

\begin{table}[htbp]
\centering
\caption{Training diagnostics for one joint-training run. Token counts are rounded. FVE is Fraction of Variance Explained; higher is better.}
\label{tab:empirical_convergence}
\begin{tabular}{rrrr}
\toprule
Step (tokens) & Level-1 reconstruction & Level-2 reconstruction & FVE \\
\midrule
5k (13M)    & $5.36\times10^{-3}$ & $4.65\times10^{-3}$ & 0.9279 \\
10k (26M)   & $4.53\times10^{-3}$ & $2.48\times10^{-3}$ & 0.9390 \\
20k (51M)   & $6.97\times10^{-4}$ & $2.00\times10^{-4}$ & 0.9906 \\
40k (102M)  & $3.86\times10^{-4}$ & $1.13\times10^{-5}$ & 0.9948 \\
80k (205M)  & $3.64\times10^{-4}$ & $1.13\times10^{-5}$ & 0.9951 \\
160k (410M) & $3.43\times10^{-4}$ & $1.07\times10^{-5}$ & 0.9954 \\
195k (500M) & $3.43\times10^{-4}$ & $1.07\times10^{-5}$ & 0.9954 \\
\bottomrule
\end{tabular}
\end{table}

Both reconstruction losses decrease during training and remain unchanged at the reported precision over the final approximately 90M tokens, indicating an empirical plateau in this run.
\endgroup

\section{Implementation Details}

\label{appsec:exp_details}

\subsection{MLLMs, Datasets, and Implementation Details}

\label{appsec:exp_setup}

\textbf{MLLM activation extraction.} We evaluate \texttt{Qwen3-VL-4B-Instruct}, \texttt{Gemma-3-4B-IT}, and \texttt{LLaVA-1.5-13B}.

For \texttt{Qwen3-VL-4B-Instruct}, we collect activations from visual transformer block 23.

For \texttt{Gemma-3-4B-IT}, we use vision tower encoder layer 26.

For \texttt{LLaVA-1.5-13B}, we extract hidden states from language backbone layer 39.

All activations are collected with the prompt ``Describe the image content accurately.''

\textbf{Datasets.} We use Color~\cite{wang2024probabilisticconceptualexplainerstrustworthy}, ImageNet~\cite{5206848}, iNaturalist~\cite{vanhorn2021benchmarkingrepresentationlearningnatural}, and COCO~\cite{lin2015microsoftcococommonobjects}.

For the Color dataset, we use its mono-semantic subset, which contains 1,000 images, and adopt an 8:2 train/test split.

For ImageNet, we randomly sample 50 images from each of the 1,000 classes in the training set, obtaining 50,000 training images, and evaluate on all 50,000 images from the ImageNet validation set.

For iNaturalist, we randomly sample 5 images from each of 10,000 species for both training and testing, forming a 1:1 train/test split.

For COCO, we randomly sample 250 images from each of the 80 predefined categories for both training and testing.

\textbf{Training and implementation.} For all methods, we fix the number of Level-1 and Level-2 neurons to
$
n_1=20{,}000,
\qquad
n_2=10{,}000,
$
respectively, to ensure a fair comparison.

All models are optimized using Adam, and all experiments are conducted on 4 NVIDIA RTX PRO 6000 GPUs.

\begingroup
\color{black}
\begin{table}[htbp]
\centering
\caption{Training configuration reported for the hierarchy experiments. The downstream steering configuration is described separately below.}
\label{tab:training_hyperparameters}
\begin{tabular}{ll}
\toprule
Item & Value \\
\midrule
Optimizer & Adam \\
Learning rate & $10^{-4}$ \\
Batch size & 2,560 activation tokens \\
Training steps & 195,312 (approximately 500M tokens) \\
Dictionary sizes & $n_1=20{,}000$, $n_2=10{,}000$ \\
TopK sparsities & $k_1=6$, $k_2=2$ \\
Soft sparsity penalties & $\lambda_1=0.1$, $\lambda_2=0.5$ \\
Hardware & 4 NVIDIA RTX PRO 6000 GPUs \\
\bottomrule
\end{tabular}
\end{table}
The soft sparsity penalties apply to variants that use these penalties.
The stacked baseline uses $k_1=6$ and $k_2=2$, satisfying $k_2<k_1$.
The downstream steering configuration is described separately in~\appref{app:steering}.
\endgroup

For baseline methods that do not natively support learning multi-level concepts, such as BatchTopK, TopK, ReLU, P-Annealing, Gated, and JumpReLU SAEs, we train two separate SAEs: one with \(n=n_1\) and one with \(n=n_2\).

The resulting Level-1 and Level-2 units are connected post hoc using the parent-assignment rule in~\appref{appsec:hms}.

For Matryoshka SAE, we use its native nested dictionary structure as the hierarchical representation.

For the stacked SAE baseline, we use the naive stacked formulation discussed in~\secref{sec:theory}; its empirical performance is reported in~\secref{sec:results}.

\subsection{Dynamic Masking of Dead Latents}

\label{appsec:dynamic_masking}

To reduce unnecessary computation in the Level-2 SAE, we use a dynamic masking strategy for inactive Level-1 atoms.

At each training step \(t\), let \(\mathcal{B}\) denote the current mini-batch.

We define the active atom set

\begin{align}
\mathcal{A}_t =
\left\{
\mathbf{w}_j \;\middle|\;
\sum_{\mathbf{x}\in\mathcal{B}}
\mathbf{1}\!\left(\left|[g_1(\mathbf{x})]_j\right|>\epsilon_0\right)
\ge 1
\right\},
\label{eq:dynamic_mask_active_set}
\end{align}

where \(\epsilon_0>0\) is a small threshold and \([g_1(\mathbf{x})]_j\) denotes the \(j\)-th entry of the Level-1 sparse code.

Instead of applying the Level-2 reconstruction loss to all \(n_1\) decoder atoms at every step, we apply it only to active atoms in \(\mathcal{A}_t\):

\begin{align}
\mathcal{L}_{2,t}
=
\frac{1}{|\mathcal{A}_t|}
\sum_{\mathbf{w}_j\in\mathcal{A}_t}
\Big(
\|\mathbf{w}_j-h_2(g_2(\mathbf{w}_j))\|_2^2
+
\lambda_2\mathcal{S}(g_2(\mathbf{w}_j))
\Big).
\label{eq:dynamic_mask_l2_loss}
\end{align}

If \(\mathcal{A}_t=\emptyset\), we skip the Level-2 update for that mini-batch.
This masking is an implementation strategy rather than a change to the model architecture.
It reduces computation by avoiding repeated Level-2 updates on atoms that are inactive for the current batch.
Across training, atoms are included in the Level-2 objective whenever they become active in sampled mini-batches.

\subsection{Stop-Gradient Ablations for the Level-2 Objective}\label{appsec:stopgrad}

In our full \modelname formulation, the Level-2 objective in~\eqnref{eq:L2} is optimized without any stop-gradient operation, allowing gradients from $\mathcal{L}_2$ to flow through the Level-1 decoder atoms $\{\mathbf{w}_j\}_{j=1}^{n_1}$. To better understand the role of this gradient coupling, we conduct an ablation study with two stop-gradient variants.

\textbf{Partial Stop-Gradient.}
In this variant, we apply a stop-gradient operator to $\mathbf{w}_j$ only in the reconstruction target of the first term in~\eqnref{eq:L2}. Concretely, the Level-2 loss is modified as
\begin{align}
\mathcal{L}_2^{\text{partial}}
= \frac{1}{n_1} \sum_{j=1}^{n_1} \Big(
\|\text{sg}[\mathbf{w}_j] - h_2(g_2(\mathbf{w}_j))\|_2^2
+ \lambda_2 \mathcal{S}(g_2(\mathbf{w}_j))
\Big),
\end{align}
where $\text{sg}[\cdot]$ denotes the stop-gradient operator. {This detaches the reconstruction target, while reconstruction and sparsity gradients can still reach the Level-1 atoms through the encoder input.}

\textbf{Bilateral Stop-Gradient.}
In this variant, we apply the stop-gradient operator to $\mathbf{w}_j$ for both terms in the Level-2 objective:
\begin{align}
\mathcal{L}_2^{\text{full}}
= \frac{1}{n_1} \sum_{j=1}^{n_1} \Big(
\|\text{sg}[\mathbf{w}_j] - h_2(g_2(\text{sg}[\mathbf{w}_j]))\|_2^2
+ \lambda_2 \mathcal{S}(g_2(\text{sg}[\mathbf{w}_j]))
\Big).
\end{align}
{This blocks Level-2 gradients into the Level-1 atoms; the atoms can still change through Level-1 training.}

\textbf{No Stop-Gradient (Full \modelname).}
By contrast, our full method does not employ any stop-gradient operation in~\eqnref{eq:L2}. This allows the Level-2 reconstruction and sparsity objectives to jointly shape the Level-1 decoder atoms, encouraging them to organize into structures that admit compact, semantically coherent higher-level abstractions. Empirically, this formulation yields the strongest performance, as reported in our ablation results.

\subsection{Baselines}

\label{appsec:baselines}

We compare \modelname{} with a diverse set of SAE baselines that represent common dictionary-learning and mechanistic-interpretability architectures:
\begin{itemize}[nosep,leftmargin=18pt]
    \item \textbf{TopK SAE}~\cite{gao2024scalingevaluatingsparseautoencoders} enforces sparsity by retaining only the top-\(K\) activations per input.
    \item \textbf{BatchTopK}~\cite{bussmann2024batchtopk} applies Top-\(K\) sparsity at the batch level, encouraging more balanced feature usage across mini-batches.
    \item \textbf{ReLU SAE}~\cite{bricken2023monosemanticity} uses a standard ReLU nonlinearity with an \(\ell_1\)-style sparsity objective.
    \item \textbf{P-Annealing}~\cite{karvonen2024measuringprogressdictionarylearning} gradually increases sparsity during training for improved optimization stability.
    \item \textbf{Gated SAE}~\cite{rajamanoharan2024improvingdictionarylearninggated} introduces learnable gates to decouple feature selection from feature magnitude estimation.
    \item \textbf{JumpReLU}~\cite{rajamanoharan2024jumpingaheadimprovingreconstruction} uses a discontinuous activation function for sharper feature selection.
    \item \textbf{Matryoshka SAE}~\cite{bussmann2025learningmultilevelfeaturesmatryoshka, zaigrajew2025interpretingcliphierarchicalsparse} uses nested dictionaries to encourage multi-scale feature separation.
    \item \textbf{Stacked SAE} stacks a second SAE on the sparse code of the first SAE, corresponding to the naive architecture analyzed in~\secref{sec:theory}.
\end{itemize}

These baselines cover hard sparsity, soft sparsity, adaptive sparsity schedules, gated feature selection, discontinuous activations, nested-prefix hierarchies, and naive stacked hierarchies.

For flat SAE baselines, parent-child structure is constructed post hoc using the rule in~\appref{appsec:hms}, so that all methods can be evaluated with the same HMS metric.

\subsection{Evaluation Metrics}
\label{appsec:hms}

\textbf{Hierarchical Mono-Semanticity (HMS).}
Existing monosemanticity evaluations mainly ask whether an individual SAE feature corresponds to a coherent semantic concept~\cite{pach2025sparseautoencoderslearnmonosemantic}.
In contrast, our goal is to evaluate \emph{hierarchical} concept discovery:
given a high-level Level-2 concept, do its assigned Level-1 children represent semantically related low-level concepts?
We therefore define Hierarchical Mono-Semanticity (HMS), which measures the semantic coherence among the Level-1 concepts grouped under the same Level-2 parent.

\paragraph{Level-1 semantic representations.}
For each Level-1 SAE neuron \(j\), we construct a semantic representation
\(\mathbf{R}_j\in\mathbb{R}^{d_e}\) from a fixed vision-backbone embedding space.
Let \(\{\mathbf{E}_i\}_{i=1}^{N}\) denote image embeddings extracted from a fixed vision backbone, where \(\mathbf{E}_i\in\mathbb{R}^{d_e}\) is the embedding of image \(i\).
Let \(v_{i,j}\) be the activation value of Level-1 neuron \(j\) on image \(i\).
We first min-max normalize the activations of neuron \(j\) across the dataset:
\begin{align}
    a_{i,j}
    =
    \frac{v_{i,j}-\min(\mathbf{v}_j)}
    {\max(\mathbf{v}_j)-\min(\mathbf{v}_j)+\epsilon},
    \qquad
    \mathbf{v}_j=\{v_{i,j}\}_{i=1}^{N},
\end{align}
where \(\epsilon>0\) is a small constant for numerical stability.
The semantic representation of neuron \(j\) is then the activation-weighted average of image embeddings:
\begin{align}
    \mathbf{R}_j
    =
    \frac{\sum_{i=1}^{N}a_{i,j}\mathbf{E}_i}
    {\sum_{i=1}^{N}a_{i,j}+\epsilon}.
\end{align}
Intuitively, \(\mathbf{R}_j\) summarizes the visual semantics of the images on which neuron \(j\) is most active.
Thus, two Level-1 neurons with similar top-activating image semantics should have nearby \(\mathbf{R}_j\)'s in the vision-backbone embedding space.

\paragraph{Parent assignment.}
Let \(\pi(j)\in\{1,\dots,n_2\}\) denote the parent assignment that maps each Level-1 neuron \(j\) to a Level-2 concept.
For \modelname{}, this assignment is induced by the Level-2 SAE.
Each Level-1 neuron \(j\) corresponds to a decoder atom
\[
\mathbf{w}_j=\mathbf{W}^{(1)}_{\mathrm{dec},j}\in\mathbb{R}^{d},
\]
which is encoded by the Level-2 encoder \(g_2\).
We assign \(j\) to the most active Level-2 latent:
\begin{align}
    \pi(j)
    =
    \argmax_{\ell\in\{1,\dots,n_2\}}
    [g_2(\mathbf{w}_j)]_\ell .
\end{align}

For baseline SAE variants that do not natively define a parent--child hierarchy, we construct \(\pi(\cdot)\) post hoc.
We train or obtain Level-1 and Level-2 units according to the corresponding baseline design, and assign each Level-1 unit to the Level-2 unit with which it has the highest co-activation frequency over the evaluation dataset.
This yields a unified parent assignment for all methods, allowing HMS to be computed with the same formula.

\paragraph{HMS definition.}
For each Level-2 concept \(k\), define its Level-1 children as
\begin{align}
    \mathcal{C}_k=\{j:\pi(j)=k\},
\end{align}
and let \(m=|\mathcal{C}_k|\).
When \(m<2\), we exclude \(k\) from aggregate HMS statistics because pairwise coherence is undefined for singleton or empty child sets.
For \(m\ge2\), the HMS score of concept \(k\) is the mean pairwise cosine similarity among the semantic representations of its Level-1 children:
\begin{align}
    \mathrm{HMS}(k)
    =
    \frac{2}{m(m-1)}
    \sum_{\substack{u,v\in\mathcal{C}_k\\u<v}}
    \frac{\mathbf{R}_{u}^{\top}\mathbf{R}_{v}}
    {\|\mathbf{R}_{u}\|_2\|\mathbf{R}_{v}\|_2}.
\end{align}
The normalization \(2/(m(m-1))\) averages over all unordered child pairs.
A high HMS score means that the children assigned to the same Level-2 parent are mutually close in semantic embedding space, indicating a coherent high-level grouping.
A low HMS score indicates that the parent mixes semantically unrelated or weakly related Level-1 concepts.

\paragraph{Aggregate statistics.}
{Let \(\mathcal{K}_{\mathrm{active}}\) be the set of Level-2 concepts with at least two alive children and nonzero semantic vectors.}
Our main paper reports \(\mathrm{HMS}_{\mathrm{mean}}\), which summarizes the overall semantic coherence of all active Level-2 concepts:
\begin{align}
    \mathrm{HMS}_{\mathrm{mean}}
    =
    \frac{1}{|\mathcal{K}_{\mathrm{active}}|}
    \sum_{k\in\mathcal{K}_{\mathrm{active}}}\mathrm{HMS}(k).
\end{align}
For completeness, the full quantitative tables in~\appref{appsec:full_hms} also report:
\begin{align}
    \mathrm{HMS}_{\min}
    &=
    \min_{k\in\mathcal{K}_{\mathrm{active}}}\mathrm{HMS}(k),\\
    \mathrm{HMS}_{\mathrm{med}}
    &=
    \operatorname{median}_{k\in\mathcal{K}_{\mathrm{active}}}\mathrm{HMS}(k),\\
    \mathrm{HMS}_{\max}
    &=
    \max_{k\in\mathcal{K}_{\mathrm{active}}}\mathrm{HMS}(k).
\end{align}
The mean is used as the primary summary metric because it reflects overall hierarchy quality across all active Level-2 concepts.
The median provides a robustness check for the typical parent concept, while the minimum and maximum characterize worst-case and best-case parent concepts.

\subsection{Steering Details}
\label{app:steering}    
\paragraph{Setup.} The MLLM we use is Qwen3-VL-4B-Instruct, with both SAEs operating on the residual stream of \texttt{visual.blocks.23}. \modelname is the two-level cascade with the number of Level-1 and Level-2 neurons to $n_1 = 20{,}000$ and $n_2 = 10{,}000$, respectively. MSAE is a global-batch-top-$k$ Matryoshka SAE with $d\!=\!30\mathrm{k}$, group sizes $[20\mathrm{k},10\mathrm{k}]$, $k\!=\!20$. Both SAEs are trained on the same $20{,}000$ COCO val2014 images.

\paragraph{Concept clusters and MSAE matching.} {We randomly select ten Level-2 \modelname{} clusters.} The L1 neurons with a concept word list were obtained by prompting the MLLM on the top-$10$ training images. MSAE has no native cluster structure, so for each \modelname cluster we discover a matched MSAE set $\mathcal{A}_{\mathrm{MSAE}}$ by selecting the $|\mathcal{A}|$ MSAE atoms with lowest average activation rank on the \emph{same} training images.

\paragraph{Steering hook.} A forward hook on \texttt{visual.blocks.23} encodes the residual activation through the SAE, overrides each neuron in $\mathcal{A}$ to a clamp value $c$, and writes back the SAE reconstruction. The clamp magnitude is $c\!=\!\pm 3\textbf{u}_s$, where $\textbf{u}_s$ is the mean post-ReLU activation of the cluster's atoms on the cluster's top-$10$ training images, computed independently for each SAE. The hook is identical for the two SAEs.

\paragraph{Methods.} For every (cluster, image) pair we generate four captions: \textbf{baseline} (no clamp), \textbf{rand1} (one alive atom clamped, deterministic seed), \textbf{best\_child} (the highest-scoring atom in $\mathcal{A}$), and \textbf{cluster} (every atom in $\mathcal{A}$ clamped jointly). Insertion uses $+3\textbf{u}_s$, suppression $-3\textbf{u}_s$.

\paragraph{Image sets.} Insertion uses $100$ random held-out images per cluster, identical between the two SAEs ($1{,}000$ rows each). Suppression filters to images where the unsteered baseline already mentions the concept {(suppression success is defined only when the target is present in the unsteered caption)}; we collect $1000$ qualifying \modelname rows and evaluate MSAE on the exact same $1000$ pairs.

\paragraph{Caption generation.} For every (cluster, image, arm) we run one greedy generation prompted with \emph{``Describe what is in this image in one short sentence.''} (\texttt{do\_sample=False}, \texttt{max\_new\_tokens=30}, no system prompt). The same prompt is used in every condition; only the activation hook changes.

\paragraph{LLM judge.} For each row, the four arm captions are anonymized to labels A/B/C/D under a row-wise random permutation. Gemini 2.5 Flash is shown the test image and the four labelled captions and answers, for each label, the binary question \emph{``does the caption semantically express the concept?''} with a one-sentence reason. 
{Results under all four judges are reported in~\tabref{tab:steering}.}

\begin{table}[!t]
\centering
\caption{
Ablation studies on ImageNet with \texttt{Qwen3-VL-4B-Instruct}.
}
\label{apptab:ablation}
\vskip -0.1cm
\setlength{\tabcolsep}{7pt}
\begin{tabular}{lcccccc}
\toprule
\textbf{Metric}
& \textbf{Partial SG}
& \textbf{Bilateral SG}
& \textbf{Stacked}
& \textbf{HAC}
& \textbf{Spectral}
& \textbf{\modelname (Full)} \\
\midrule
\textbf{HMS}$_{\min}$
& 0.0263 & 0.0101 & \textbf{0.2765} & 0.0000 & 0.0000 & \underline{0.0380} \\
\textbf{HMS}$_{\mathrm{med}}$
& 0.8280 & 0.7012 & 0.3367 & \underline{0.8587} & 0.6906 & \textbf{0.9999} \\
\textbf{HMS}$_{\max}$
& \textbf{0.9999} & \textbf{0.9999} & 0.3670 & 0.9893 & \underline{0.9907} & \textbf{0.9999} \\
\textbf{HMS}$_{\mathrm{mean}}$
& \underline{0.7099} & 0.6333 & 0.3322 & 0.7025 & 0.6490 & \textbf{0.7776} \\
\bottomrule
\end{tabular}
\renewcommand{\arraystretch}{1.0}
\vskip -0.25cm
\end{table}

\section{Full Quantitative Results}\label{appsec:more_quantitatively}
\label{appsec:full_hms}

\textbf{Full Results on HMS Scores.} \tabref{tab:results_qwen}$\sim$\ref{tab:results_llava} show the full results for all HMS scores across different MLLMs and different datasets. 

\textbf{Ablation Studies.}
\tabref{apptab:ablation} evaluates which components are responsible for the gains of \modelnames{}: (1) Restricting gradient flow between the two SAE levels hurts performance: both partial and bilateral stop-gradient variants (\textbf{Partial SG} and \textbf{Bilateral SG}, details in~\appref{appsec:stopgrad}) lead to lower {HMS} scores compared with full joint training (\textbf{{\modelname (Full)}}), indicating that the Level-2 objective should shape the Level-1 atoms rather than merely cluster a fixed dictionary. (2) The Stacked SAE (\textbf{Stacked}) performs substantially worse, supporting our theoretical argument that naively re-compressing sparse activation codes is poorly suited for multi-level concept learning. {(3) Replacing the Level-2 SAE with post-hoc clustering on Level-1 decoder weight columns $\w_j$, including Hierarchical Agglomerative Clustering (\textbf{HAC}) and Spectral Clustering (\textbf{SC}), also underperforms our full \modelname.
This suggests that the learned Level-2 sparse abstraction is more effective than post-hoc clustering.} 


\begin{table}[H]
\centering
\caption{HMS scores for on \texttt{Qwen3-VL-4B-Instruct}. We mark the best results in \textbf{bold} and the second best with \underline{underline}.}
\label{tab:results_qwen}
\resizebox{\linewidth}{!}{%
\begin{tabular}{llccccccccc}
\toprule
\textbf{Dataset} & \textbf{Metric} 
& \textbf{BatchTopK} 
& \textbf{TopK} 
& \textbf{ReLU} 
& \textbf{P-Annealing} 
& \textbf{Gated} 
& \textbf{JumpReLU} 
& \textbf{Matryoshka} 
& \textbf{Stacked} 
& \textbf{\modelname{} (Ours)} \\
\midrule

\multirow{4}{*}{Color}
& \textbf{HMS}$_{\min}$ 
& 0.1893 & 0.1015 & 0.2048 & 0.3830 & 0.2947 & 0.1030 & \textbf{0.9710} & 0.4585 & \underline{0.5797} \\
& \textbf{HMS}$_{\mathrm{med}}$ 
& 0.4202 & 0.1076 & 0.2125 & 0.5451 & 0.5200 & 0.1267 & \underline{0.9990} & 0.6279 & \textbf{0.9999} \\
& \textbf{HMS}$_{\max}$ 
& 0.8704 & 0.1581 & 0.2250 & 0.7073 & 0.9900 & 0.3344 & \underline{0.9995} & 0.9964 & \textbf{0.9999} \\
& \textbf{HMS}$_{\mathrm{mean}}$ 
& 0.4207 & 0.1224 & 0.2141 & 0.5450 & 0.6547 & 0.1465 & \underline{0.9775} & 0.6012 & \textbf{0.9825} \\

\midrule

\multirow{4}{*}{ImageNet}
& \textbf{HMS}$_{\min}$ 
& \underline{0.1341} & 0.1318 & 0.0572 & 0.0693 & 0.0311 & 0.0597 & 0.1104 & \textbf{0.2765} & 0.0380 \\
& \textbf{HMS}$_{\mathrm{med}}$ 
& 0.1413 & 0.1387 & 0.6041 & \underline{0.7984} & 0.5612 & 0.0680 & 0.6146 & 0.3367 & \textbf{0.9999} \\
& \textbf{HMS}$_{\max}$ 
& 0.6092 & 0.2124 & \underline{0.9772} & 0.9334 & 0.9289 & 0.0990 & 0.9979 & 0.3670 & \textbf{0.9999} \\
& \textbf{HMS}$_{\mathrm{mean}}$ 
& 0.2564 & 0.1554 & 0.5624 & \underline{0.6393} & 0.5613 & 0.0756 & 0.5843 & 0.3322 & \textbf{0.7776} \\

\midrule

\multirow{4}{*}{COCO}
& \textbf{HMS}$_{\min}$ 
& 0.1116 & 0.1158 & 0.1118 & \underline{0.1204} & 0.0867 & 0.0894 & 0.0334 & 0.1156 & \textbf{0.1900} \\
& \textbf{HMS}$_{\mathrm{med}}$ 
& 0.1220 & 0.1436 & 0.1255 & 0.1257 & 0.1923 & 0.1171 & \underline{0.4598} & 0.1468 & \textbf{0.9999} \\
& \textbf{HMS}$_{\max}$ 
& 0.1650 & 0.1864 & 0.9471 & 0.1310 & 0.8634 & 0.1554 & \underline{0.9978} & 0.1731 & \textbf{0.9999} \\
& \textbf{HMS}$_{\mathrm{mean}}$ 
& 0.1270 & 0.1514 & 0.3275 & 0.1257 & 0.2809 & 0.1206 & \underline{0.4614} & 0.1443 & \textbf{0.9795} \\

\midrule

\multirow{4}{*}{iNaturalist}
& \textbf{HMS}$_{\min}$ 
& 0.0851 & 0.0838 & \textbf{0.1225} & 0.1071 & 0.1041 & 0.0867 & 0.0892 & \underline{0.1099} & 0.0323 \\
& \textbf{HMS}$_{\mathrm{med}}$ 
& 0.2436 & 0.2648 & 0.2001 & 0.1400 & 0.1928 & 0.1003 & \underline{0.5983} & 0.2254 & \textbf{0.9738} \\
& \textbf{HMS}$_{\max}$ 
& 0.8308 & 0.9257 & 0.2573 & \textbf{0.9999} & 0.8418 & 0.1139 & 0.9931 & 0.2850 & \textbf{0.9999} \\
& \textbf{HMS}$_{\mathrm{mean}}$ 
& 0.3182 & 0.3254 & 0.1934 & 0.4157 & 0.2611 & 0.1016 & \underline{0.5879} & 0.2307 & \textbf{0.7412} \\

\bottomrule
\end{tabular}%
}
\vskip -0.4cm
\end{table}

\begin{table}[H]
\centering
\caption{HMS scores for on \texttt{Gemma-3-4B-IT}. We mark the best results in \textbf{bold} and the second best with \underline{underline}.}
\label{tab:results_gemma}
\resizebox{\linewidth}{!}{%
\begin{tabular}{llccccccccc}
\toprule
\textbf{Dataset} & \textbf{Metric} 
& \textbf{BatchTopK} 
& \textbf{TopK} 
& \textbf{ReLU} 
& \textbf{P-Annealing} 
& \textbf{Gated} 
& \textbf{JumpReLU} 
& \textbf{Matryoshka} 
& \textbf{Stacked} 
& \textbf{\modelname{} (Ours)} \\
\midrule

\multirow{4}{*}{Color}
& \textbf{HMS}$_{\min}$ 
& 0.4740 & 0.1876 & 0.2001 & 0.5925 & 0.1106 & 0.1264 & \underline{0.9750} & 0.5901 & \textbf{0.9814} \\
& \textbf{HMS}$_{\mathrm{med}}$ 
& 0.6412 & 0.3885 & 0.2237 & 0.6584 & 0.6123 & 0.1301 & \underline{0.9845} & 0.6612 &\textbf{0.9936} \\
& \textbf{HMS}$_{\max}$ 
& 0.8850 & 0.9991 & 0.2425 & 0.7243 & 0.8146 & 0.2523 & \underline{0.9995} & 0.9911 &\textbf{0.9997} \\
& \textbf{HMS}$_{\mathrm{mean}}$ 
& 0.6667 & 0.4452 & 0.2221 & 0.6602 & 0.5662 & 0.1597 & \underline{0.9849} & 0.7097& \textbf{0.9929} \\

\midrule

\multirow{4}{*}{ImageNet}
& \textbf{HMS}$_{\min}$ 
& 0.1448 & 0.1294 & 0.0792 & 0.0125 & -0.0025 & 0.0238 & 0.2010 & \underline{0.2975} & \textbf{0.4155} \\
& \textbf{HMS}$_{\mathrm{med}}$ 
& 0.1589 & 0.1351 & 0.0798 & 0.0899 & 0.0883 & 0.0867 & \underline{0.7532} & 0.4991 &\textbf{0.8375} \\
& \textbf{HMS}$_{\max}$ 
& 0.1729 & 0.1407 & 0.0805 & 0.2312 & 0.7234 & 0.9608 & \underline{0.9926} & 0.7007 &\textbf{0.9998} \\
& \textbf{HMS}$_{\mathrm{mean}}$ 
& 0.1605 & 0.1338 & 0.0797 & 0.1501 & 0.2564 & 0.1571 & \underline{0.6874} & 0.4879 &\textbf{0.7750} \\

\midrule

\multirow{4}{*}{COCO}
& \textbf{HMS}$_{\min}$ 
& 0.1930 & 0.0527 & \textbf{0.4659} & 0.1384 & 0.1956 & 0.1312 & 0.2187 & \underline{0.2205} &0.0293 \\
& \textbf{HMS}$_{\mathrm{med}}$ 
& 0.2802 & 0.0530 & 0.4659 & 0.1596 & \underline{0.8284} & 0.1736 & 0.7250 & 0.5083 & \textbf{0.9999} \\
& \textbf{HMS}$_{\max}$ 
& 0.3675 & 0.0533 & 0.4659 & 0.1809 & 0.9736 & 0.1952 & \textbf{0.9999} & 0.6769 & \textbf{0.9999} \\
& \textbf{HMS}$_{\mathrm{mean}}$ 
& 0.2802 & 0.0530 & 0.4659 & 0.1596 & 0.6659 & 0.1667 & \underline{0.7020} & 0.4124 & \textbf{0.8907} \\

\midrule

\multirow{4}{*}{iNaturalist}
& \textbf{HMS}$_{\min}$ 
& 0.2109 & 0.1572 & 0.0713 & 0.1333 & -0.0020 & 0.1010 & \textbf{0.7186} & 0.2472 & \underline{0.4012} \\
& \textbf{HMS}$_{\mathrm{med}}$ 
& 0.5028 & 0.2299 & 0.0892 & 0.1481 & 0.0959 & 0.1435 & \underline{0.8863} & 0.4557 & \textbf{0.9607} \\
& \textbf{HMS}$_{\max}$ 
& 0.7947 & 0.3027 & 0.1072 & 0.1629 & \textbf{0.9999} & \textbf{0.9999} & 0.9879 & 0.6642 & 0.9763 \\
& \textbf{HMS}$_{\mathrm{mean}}$ 
& 0.5028 & 0.2299 & 0.0892 & 0.1481 & 0.2254 & 0.4149 & \underline{0.8687} & 0.4601 & \textbf{0.9128} \\

\bottomrule
\end{tabular}%
}
\vskip -0.3cm
\end{table}

\begin{table}[H]
\centering
\caption{HMS scores for on \texttt{LLaVA-1.5-13B}. We mark the best results in \textbf{bold} and the second best with \underline{underline}.}
\label{tab:results_llava}
\resizebox{\linewidth}{!}{%
\begin{tabular}{llccccccccc}
\toprule
\textbf{Dataset} & \textbf{Metric} 
& \textbf{BatchTopK} 
& \textbf{TopK} 
& \textbf{ReLU} 
& \textbf{P-Annealing} 
& \textbf{Gated} 
& \textbf{JumpReLU} 
& \textbf{Matryoshka} 
& \textbf{Stacked} 
& \textbf{\modelname{} (Ours)} \\
\midrule

\multirow{4}{*}{Color}
& \textbf{HMS}$_{\min}$ 
& 0.2481 & 0.1104 & 0.0587 & 0.1968 & 0.0979 & 0.2013 & \textbf{0.9273} & 0.2536 & \underline{0.5672} \\
& \textbf{HMS}$_{\mathrm{med}}$ 
& 0.3220 & 0.2396 & 0.2475 & 0.3250 & 0.3525 & 0.3088 & \underline{0.9377} & 0.6788 & \textbf{0.9481} \\
& \textbf{HMS}$_{\max}$ 
& 0.4461 & 0.4154 & \textbf{0.9999} & 0.5128 & \textbf{0.9999} & \textbf{0.9999} & \textbf{0.9999}  & 0.9953 & \textbf{0.9999} \\
& \textbf{HMS}$_{\mathrm{mean}}$ 
& 0.3307 & 0.2421 & 0.3170 & 0.3357 & 0.4324 & 0.4337 & \underline{0.9618} & 0.5767 & \textbf{0.9705} \\

\midrule

\multirow{4}{*}{ImageNet}
& \textbf{HMS}$_{\min}$ 
& 0.1043 & 0.0796 & 0.0168 & 0.0113 & 0.1414 & 0.0289 & \textbf{0.2260} & \underline{0.1631} & 0.0942 \\
& \textbf{HMS}$_{\mathrm{med}}$ 
& 0.1751 & 0.2176 & 0.0814 & 0.4980 & 0.1976 & 0.2757 & \underline{0.7792} & 0.1990 & \textbf{0.9650} \\
& \textbf{HMS}$_{\max}$ 
& 0.6983 & 0.9541 & 0.5679 & 0.8122 & 0.2811 & 0.9870 & \underline{0.9949} & 0.6534 & \textbf{0.9999} \\
& \textbf{HMS}$_{\mathrm{mean}}$ 
& 0.1913 & 0.2458 & 0.1100 & 0.4934 & 0.1945 & 0.4099 & \underline{0.7471} & 0.2304 & \textbf{0.8958} \\

\midrule

\multirow{4}{*}{COCO}
& \textbf{HMS}$_{\min}$ 
& 0.0459 & 0.0855 & 0.0766 & 0.0991 & 0.0895 & \textbf{0.1173} & 0.0482 & \underline{0.1166} & 0.0624 \\
& \textbf{HMS}$_{\mathrm{med}}$ 
& 0.3296 & 0.3382 & 0.2535 & 0.2706 & 0.0895 & 0.2068 & \underline{0.8646} & 0.2031 & \textbf{0.9970} \\
& \textbf{HMS}$_{\max}$ 
& 0.9918 & 0.9954 & 0.9969 & 0.9758 & 0.0895 & 0.5017 & \underline{0.9998} & 0.7896 & \textbf{1.0000} \\
& \textbf{HMS}$_{\mathrm{mean}}$ 
& 0.3966 & 0.3963 & 0.3206 & 0.3339 & 0.0895 & 0.2368 & \underline{0.8107} & 0.2843 & \textbf{0.9405} \\

\midrule

\multirow{4}{*}{iNaturalist}
& \textbf{HMS}$_{\min}$ 
& 0.0902 & 0.1322 & 0.0688 & 0.1126 & 0.0880 & 0.0059 & 0.1117 & \underline{0.2093} & \textbf{0.3531} \\
& \textbf{HMS}$_{\mathrm{med}}$ 
& 0.2984 & 0.2865 & 0.3124 & 0.2419 & 0.2929 & 0.2122 & \underline{0.7855} & 0.2286 & \textbf{0.9871} \\
& \textbf{HMS}$_{\max}$ 
& 0.9420 & 0.8361 & 0.9883 & 0.6619 & 0.9958 & 0.9956 & \underline{0.9969} & 0.3690 & \textbf{0.9992} \\
& \textbf{HMS}$_{\mathrm{mean}}$ 
& 0.3330 & 0.3156 & 0.3644 & 0.2655 & 0.3633 & 0.2575 & \underline{0.7586} & 0.2598 & \textbf{0.7798} \\

\bottomrule
\end{tabular}%
}
\vskip -0.5cm
\end{table}

\section{More Qualitative Results}\label{appsec:more_qualitative}
In~\figref{fig:more_qualitative}, we provide more qualitative results on our \modelname's learned Level-2 and Level-1 concepts from MLLMs, compared with Matryoshka SAEs.
\begin{figure}[H]
  \centering
  \includegraphics[width=\linewidth]{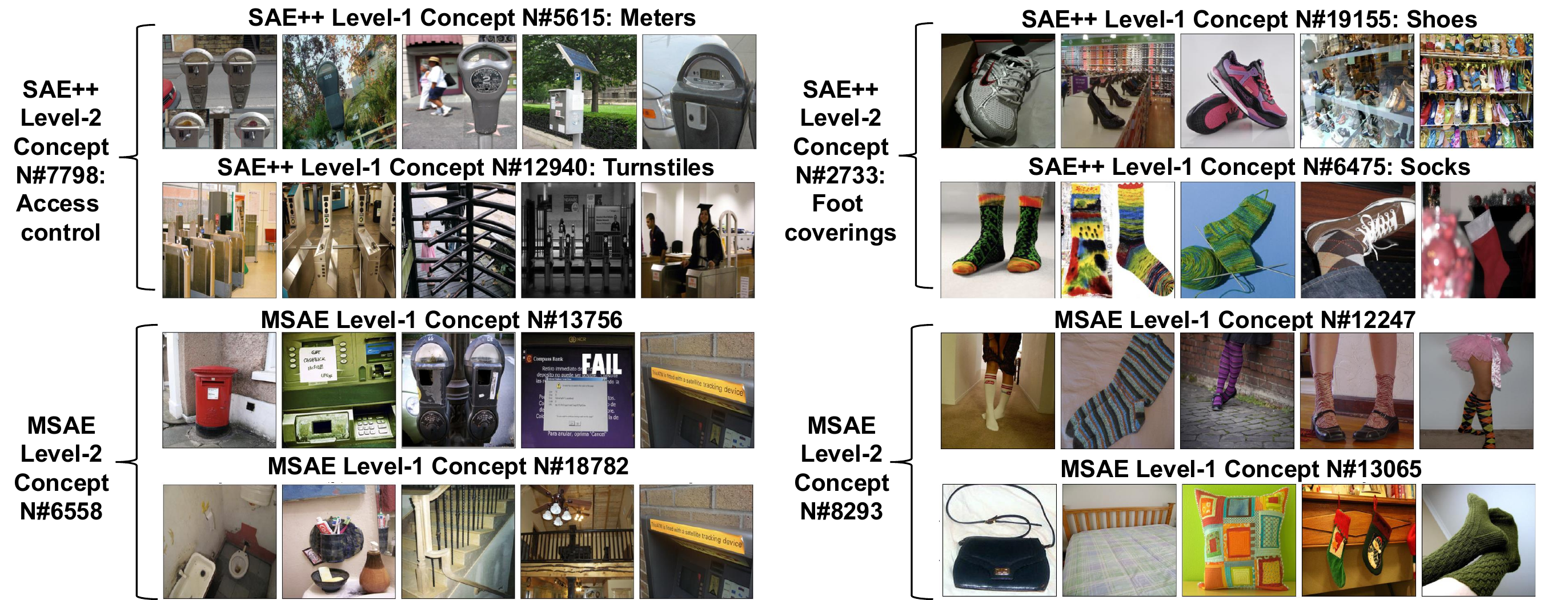}
   \caption{More qualitative results on our \modelname's learned Level-1 and Level-2 concepts from MLLMs, compared with Matryoshka SAEs.}\label{fig:more_qualitative_1}
\end{figure}

\begin{figure}[H]
  \centering
  \includegraphics[width=\linewidth]{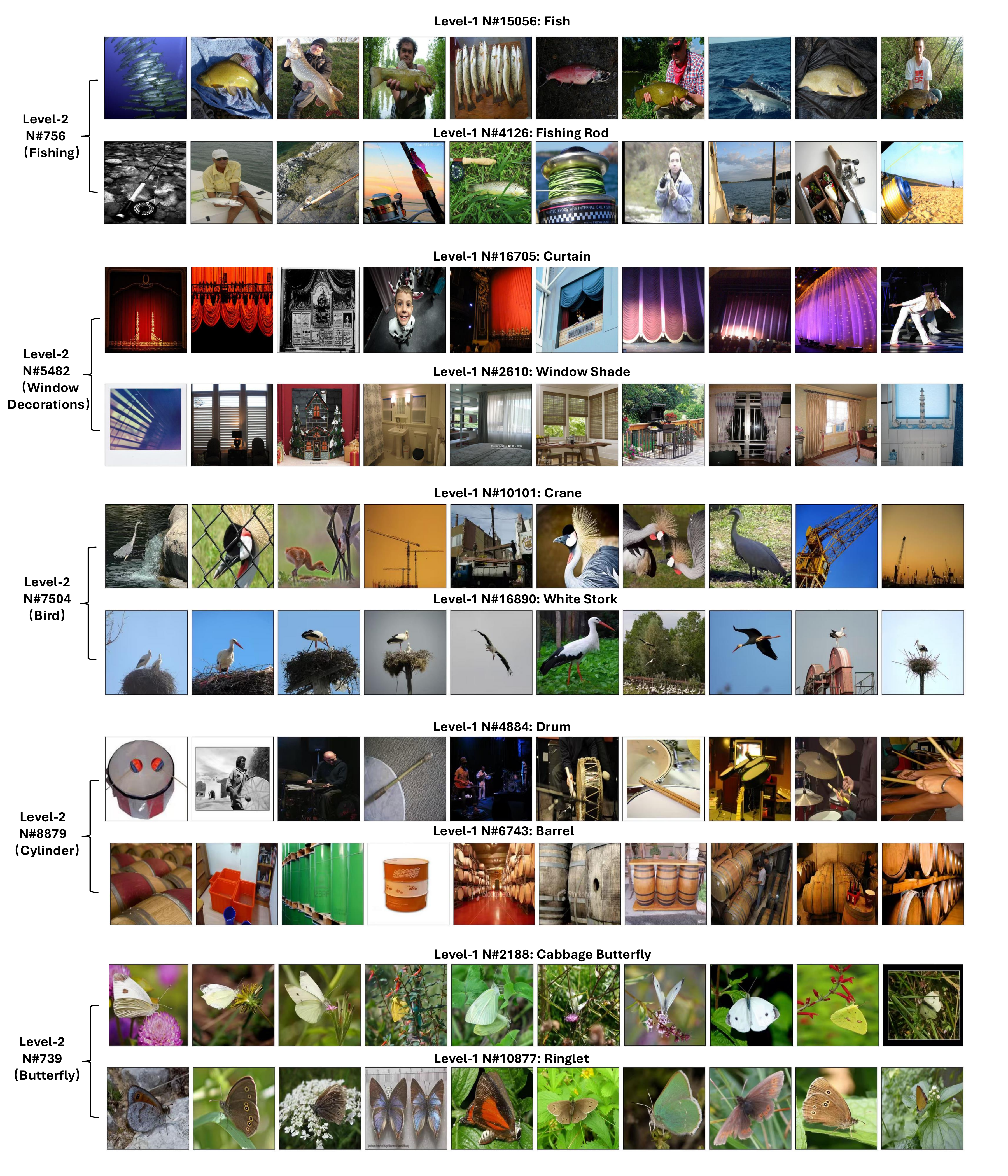}
   \caption{More qualitative results on our \modelname's learned Level-1 and Level-2 concepts from MLLMs.}\label{fig:more_qualitative}
\end{figure}

\clearpage



%% file: checklist.tex
\section*{NeurIPS Paper Checklist}

\begin{enumerate}

\item {\bf Claims}
    \item[] Question: Do the main claims made in the abstract and introduction accurately reflect the paper's contributions and scope?
    \item[] Answer: \answerYes{} 
    \item[] Justification: See Sections~\ref{sec:experiments} and~\ref{sec:results}.

\item {\bf Limitations}
    \item[] Question: Does the paper discuss the limitations of the work performed by the authors?
    \item[] Answer: \answerYes{} 
    \item[] Justification: See~\appref{appsec:limit}.

\item {\bf Theory assumptions and proofs}
    \item[] Question: For each theoretical result, does the paper provide the full set of assumptions and a complete (and correct) proof?
    \item[] Answer: \answerYes{} 
    \item[] Justification: See~\appref{appsec:theory}.

    \item {\bf Experimental result reproducibility}
    \item[] Question: Does the paper fully disclose all the information needed to reproduce the main experimental results of the paper to the extent that it affects the main claims and/or conclusions of the paper (regardless of whether the code and data are provided or not)?
    \item[] Answer: \answerYes{} 
    \item[] Justification:  See~\appref{appsec:exp_details}.

\item {\bf Open access to data and code}
    \item[] Question: Does the paper provide open access to the data and code, with sufficient instructions to faithfully reproduce the main experimental results, as described in supplemental material?
    \item[] Answer: \answerYes{} 
    \item[] Justification: Code is available at \url{https://github.com/Wang-ML-Lab/sae-plus-plus}. The benchmark datasets are cited in the paper, and implementation and evaluation details are provided in Appendix~\ref{appsec:exp_details}.

\item {\bf Experimental setting/details}
    \item[] Question: Does the paper specify all the training and test details (e.g., data splits, hyperparameters, how they were chosen, type of optimizer) necessary to understand the results?
    \item[] Answer: \answerYes{} 
    \item[] Justification: Yes, see Section. \ref{appsec:exp_details}.

\item {\bf Experiment statistical significance}
    \item[] Question: Does the paper report error bars suitably and correctly defined or other appropriate information about the statistical significance of the experiments?
    \item[] Answer: \answerYes{} 
    \item[] Justification: Appendix~\ref{appsec:steering_variability} reports means and standard deviations across ten target concepts and paired t-tests on matched concept-level steering success rates. Appendix~\ref{appsec:off_target} reports means and standard deviations across twenty unrelated concepts for the off-target evaluation.

\item {\bf Experiments compute resources}
    \item[] Question: For each experiment, does the paper provide sufficient information on the computer resources (type of compute workers, memory, time of execution) needed to reproduce the experiments?
    \item[] Answer: \answerYes{} 
    \item[] Justification: See~\appref{appsec:exp_details}.
    
\item {\bf Code of ethics}
    \item[] Question: Does the research conducted in the paper conform, in every respect, with the NeurIPS Code of Ethics \url{https://neurips.cc/public/EthicsGuidelines}?
    \item[] Answer: \answerYes{} 
    \item[] Justification: We follow the NeurIPS Code of Ethics.

\item {\bf Broader impacts}
    \item[] Question: Does the paper discuss both potential positive societal impacts and negative societal impacts of the work performed?
    \item[] Answer: \answerNA{} 
    \item[] Justification: Our model used public LLMs and datasets, and we can not foresee any potential societal impact.
    
\item {\bf Safeguards}
    \item[] Question: Does the paper describe safeguards that have been put in place for responsible release of data or models that have a high risk for misuse (e.g., pre-trained language models, image generators, or scraped datasets)?
    \item[] Answer: \answerNA{} 
    \item[] Justification:  No such risks.

\item {\bf Licenses for existing assets}
    \item[] Question: Are the creators or original owners of assets (e.g., code, data, models), used in the paper, properly credited and are the license and terms of use explicitly mentioned and properly respected?
    \item[] Answer: \answerYes{} 
    \item[] Justification:  We did it properly. 

\item {\bf New assets}
    \item[] Question: Are new assets introduced in the paper well documented and is the documentation provided alongside the assets?
    \item[] Answer: \answerNA{} 
    \item[] Justification: No new assets.

\item {\bf Crowdsourcing and research with human subjects}
    \item[] Question: For crowdsourcing experiments and research with human subjects, does the paper include the full text of instructions given to participants and screenshots, if applicable, as well as details about compensation (if any)? 
    \item[] Answer: \answerNA{} 
    \item[] Justification: No human subjects.

\item {\bf Institutional review board (IRB) approvals or equivalent for research with human subjects}
    \item[] Question: Does the paper describe potential risks incurred by study participants, whether such risks were disclosed to the subjects, and whether Institutional Review Board (IRB) approvals (or an equivalent approval/review based on the requirements of your country or institution) were obtained?
    \item[] Answer: \answerNA{} 
    \item[] Justification: The study does not involve crowdsourcing or research with human subjects.

\item {\bf Declaration of LLM usage}
    \item[] Question: Does the paper describe the usage of LLMs if it is an important, original, or non-standard component of the core methods in this research? Note that if the LLM is used only for writing, editing, or formatting purposes and does \emph{not} impact the core methodology, scientific rigor, or originality of the research, declaration is not required.
    \item[] Answer: \answerYes{} 
    \item[] Justification: LLMs are used to evaluate target-concept presence in steered captions and off-target concept presence. Their roles and evaluation protocols are described in Appendices~\ref{app:steering} and~\ref{appsec:multi_judge}.

\end{enumerate}